\documentclass[11pt]{article}

\usepackage{graphicx} 

\usepackage[square]{natbib}

\usepackage{wrapfig}
\usepackage{caption}
\usepackage{subcaption}
\usepackage{booktabs}
\usepackage{amsmath}

\usepackage{hyperref}

\usepackage{fullpage}

\usepackage{algorithm}
\usepackage{algpseudocode}

\newcommand{\noiselb}{\nu} 
\newcommand{\esttol}{\Delta} 
\newcommand{\noiseub}{\Gamma} 
\newcommand{\eigone}{\Lambda_1} 
\newcommand{\eigtwo}{\Lambda_2} 

\usepackage[usenames,dvipsnames]{xcolor}

\newcommand{\matt}[1]{\textcolor{red}{Matt: #1}}

 \renewcommand{\matt}[1]{}

\newcommand{\Es}{\mathcal E}
\newcommand{\Fs}{\mathcal F}

\newcommand{\Xs}{\mathcal X}

\newcommand{\reals}{\mathbb R}

\usepackage{amsmath,amsfonts}
\DeclareMathOperator{\E}{\mathbb{E}}

\DeclareMathOperator{\trace}{\mathrm{tr}}
\DeclareMathOperator{\diag}{\mathrm{diag}}

\DeclareMathOperator{\Cov}{\mathrm{Cov}}

\newcommand{\argmin}{\operatornamewithlimits{argmin}}

\usepackage{amsthm}
\newtheorem{theorem}{Theorem}[section]
\newtheorem{lemma}{Lemma}[section]
\newtheorem{corollary}{Corollary}[section]

\newtheorem{proposition}{Proposition}[section]
\theoremstyle{definition}
\newtheorem{definition}{Definition}[section]

\hypersetup{
    colorlinks,
    linkcolor={red!50!black},
    citecolor={blue!50!black},
    urlcolor={blue!80!black}
}

\title{Escaping Saddle Points with Adaptive Gradient Methods}

\author{Matthew Staib\footnote{Based on work performed at Google Research, New York.} \\ MIT EECS \\ mstaib@mit.edu \and 
Sashank Reddi \\ Google Research, New York \\ sashank@google.com \and
Satyen Kale \\ Google Research, New York \\ satyenkale@google.com \and
Sanjiv Kumar \\ Google Research, New York \\ sanjivk@google.com \and
Suvrit Sra \\ MIT EECS \\ suvrit@mit.edu}

\date{}

\begin{document}

\maketitle

\begin{abstract} 
Adaptive methods such as Adam and RMSProp are widely used in deep learning but are not well understood. In this paper, we seek a crisp, clean and precise characterization of their behavior in nonconvex settings. To this end, we first provide a novel view of adaptive methods as preconditioned SGD, where the preconditioner is estimated in an online manner. By studying the preconditioner on its own, we elucidate its purpose: it rescales the stochastic gradient noise to be isotropic near stationary points, which helps escape saddle points. Furthermore, we show that adaptive methods can efficiently estimate the aforementioned preconditioner. By gluing together these two components, we provide the first (to our knowledge) second-order convergence result for any adaptive method. The key insight from our analysis is that, compared to SGD, adaptive methods escape saddle points faster, and can converge faster overall to second-order stationary points.
\end{abstract} 


\section{Introduction}
Stochastic first-order methods are the algorithms of choice for training deep networks, or more generally optimization problems of the form $\argmin_x \mathbb{E}_z[f(x, z)]$. While vanilla stochastic gradient descent (SGD) is still 
the most popular such algorithm,
there has been much recent interest in  adaptive methods that adaptively change learning rates for each parameter.
This is a very old idea, e.g.~\citep{jacobs1988increased};
modern variants such as Adagrad~\citep{Duchi:2011:ASM:1953048.2021068,DBLP:conf/colt/McMahanS10} Adam~\citep{kingma2014adam} and RMSProp~\citep{tieleman2012lecture} are widely used in deep learning due to their good empirical performance.

Adagrad uses the square root of the sum of the outer product of the past gradients to achieve adaptivity. In particular, at time step $t$, Adagrad updates the parameters in the following manner:
$$
x_{t+1} = x_t - G_t^{-1/2} g_t,
$$
where $g_t$ is a noisy stochastic gradient at $x_t$ and $G_t = \sum_{i=1}^t g_i g_i^T$. More often, a diagonal version of Adagrad is used due to practical considerations, which effectively yields a per parameter learning rate. In the convex setting, Adagrad achieves provably good performance, especially when the gradients are sparse. Although Adagrad works well in sparse convex settings, its performance appears to deteriorate in (dense) nonconvex settings. This performance degradation is often attributed to the rapid decay of the learning rate in Adagrad over time, which is a consequence of rapid increase in eigenvalues of the matrix $G_t$.

To tackle this issue, variants of Adagrad such as  Adam and RMSProp have been proposed, which replace the sum of the outer products with an exponential moving average i.e., $G_t = (1 - \beta) \sum_{i=1}^t \beta^{t-i} g_i g_i^T$ for some constant $\beta \in (0,1)$. This connection with Adagrad is often used to justify the design of Adam and RMSProp (e.g.~\citep{Goodfellow-et-al-2016}).  Although this connection is simple and appealing, it is clearly superficial. For instance, while learning rates in Adagrad decrease monotonically, it is not necessarily the case with Adam or RMSProp as shown recently in \citet{reddi2018convergence}, leading to their non-convergence in even simple convex settings. Thus, a principled understanding of these adaptive methods is largely missing.

In this paper, we introduce a much simpler way of thinking about adaptive methods such as Adam and RMSProp. Roughly, adaptive methods try to precondition SGD by some matrix $A$, e.g. when $A$ is diagonal, $A_{ii}$ corresponds to the effective stepsize for coordinate $i$.  For some choices of $A$ the algorithms do not have oracle access to $A$, but instead form an estimate $\hat A \approx A$.
We separate out these two steps, by 1) giving convergence guarantees for an idealized setting where we have access to $A$, then 2) proving bounds on the quality of the estimate $\hat A$.  Our approach makes it possible to effectively intuit about the algorithms, prove convergence guarantees (including second-order convergence), and give insights about how to choose algorithm parameters. It also leads to a number of surprising results, including an understanding of why the~\citet{reddi2018convergence} counterexample is hard for adaptive methods, why adaptive methods tend to escape saddle points faster than SGD (observed in~\citep{pmlr-v84-reddi18a}), insights into how to tune Adam's parameters, and (to our knowledge) the first \emph{second-order} convergence proof for any adaptive method. 


\paragraph{Contributions:} In addition to the aforementioned novel viewpoint, we also make the following key contributions:
\begin{itemize}
  \item We develop a new approach for analyzing convergence of adaptive methods leveraging the preconditioner viewpoint and by way of disentangling estimation from the behavior of the \emph{idealized} preconditioner.
  \item We provide \emph{second-order convergence} results for adaptive methods, and as a byproduct, first-order convergence results. To the best of our knowledge, ours is the first work to show second order convergence for any adaptive method.
  \item We provide theoretical insights on how adaptive methods escape saddle points quickly. In particular, we show that the preconditioner used in adaptive methods leads to isotropic noise near stationary points, which helps escape saddle points faster.
  \item Our analysis also provides practical suggestions for tuning the exponential moving average parameter $\beta$.
\end{itemize}

\subsection{Related work}

There is an immense amount of work studying nonconvex optimization for machine learning, which is too much to discuss here in detail. Thus, we only briefly discuss two lines of work that are most relevant to our paper here.
First, the recent work e.g.~\citep{chen2018convergence,reddi2018convergence,zou2018sufficient} to understand and give theoretical guarantees for adaptive methods such as Adam and RMSProp. Second, the technical developments in using first-order algorithms to achieve nonconvex second-order convergence (see Definition~\ref{def:second-order-stationary}) e.g.~\citep{pmlr-v40-Ge15,NIPS2018_7629,pmlr-v70-jin17a,pmlr-v49-lee16}. \matt{other things to add to random assortment of cites here?} 

\paragraph{Nonconvex convergence of adaptive methods.}
Many recent works have investigated convergence properties of adaptive methods.  However, to our knowledge, all these results either require convexity or show only first-order convergence to stationary points. 
\citet{reddi2018convergence} showed non-convergence of Adam and RMSProp in simple convex settings and provided a variant of Adam, called AMSGrad, with guaranteed convergence in the convex setting; \citet{zhou2018convergence} generalized this to a nonconvex first-order convergence result. \citet{zaheer} showed first-order convergence of Adam when the batch size grows over time.
\citet{chen2018convergence} bound the nonconvex convergence rate for a large family of Adam-like algorithms, but they essentially need to assume the effective stepsize is well-behaved (as in AMSGrad).  \citet{agarwal2018case} give a convex convergence result for a full-matrix version of RMSProp, which they extend to the nonconvex case via iteratively optimizing convex functions. Their algorithm uses a fixed sliding window instead of an exponential moving average.
\citet{pmlr-v70-mukkamala17a} prove improved convergence bounds for Adagrad in the online strongly convex case; they prove similar results for RMSProp, but only in a regime where it is essentially the same as Adagrad.
\citet{ward2018adagrad} give a nonconvex convergence result for a variant of Adagrad which employs an adaptively decreasing single learning rate (not per-parameter). 
\citet{zou2018sufficient} give sufficient conditions for first-order convergence of Adam.

\paragraph{Nonconvex second order convergence of first order methods.}
Starting with~\citet{pmlr-v40-Ge15} there has been a resurgence in interest in giving first-order algorithms that find \emph{second} order stationary points of nonconvex objectives, where the gradient is small and the Hessian is nearly positive semidefinite.
Most other results in this space operate in the deterministic setting where we have exact gradients, with carefully injected isotropic noise to escape saddle points.
\citet{levy2016power} show improved results for normalized gradient descent.
Some algorithms rely on Hessian-vector products instead of pure gradient information e.g.~\citep{Agarwal:2017:FAL:3055399.3055464,doi:10.1137/17M1114296}; it is possible to reduce Hessian-vector based algorithms to gradient algorithms~\citep{NIPS2018_7797,NIPS2018_7629}.
\citet{pmlr-v70-jin17a} improve the dependence on dimension to polylogarithmic.
\citet{mokhtari2018escaping} work towards adapting these techniques for constrained optimization.
Most relevant to our work is that of~\citet{pmlr-v80-daneshmand18a}, who prove convergence of SGD with better rates than~\citet{pmlr-v40-Ge15}.
Our work differs in that we provide second-order results for \emph{preconditioned} SGD.

\section{Notation and definitions}
The objective function is $f$, and the gradient and Hessian of $f$ are $\nabla f$ and $H = \nabla^2 f$, respectively.
Denote by $x_t \in \reals^d$ the iterate at time $t$, by $g_t$ an unbiased stochastic gradient at $x_t$ and by $\nabla_t$ the expected gradient at $t$.
The matrix $G_t$ refers to $\E[g_t g_t^T]$.
Denote by $\lambda_\mathrm{max}(G)$ and $\lambda_\mathrm{min}(G)$ the largest and smallest eigenvalues of $G$, and 
$\kappa(G)$ is the condition number $\lambda_\mathrm{max}(G) / \lambda_\mathrm{min}(G)$ of $G$.
For a vector $v$, its elementwise $p$-th power is written $v^p$. 
The objective $f(x)$ has global minimizer $x^*$, and we write $f^* = f(x^*)$.
The Euclidean norm of a vector $v$ is written as $\lVert v \rVert$, while for a matrix $M$, $\lVert M \rVert$ refers to the operator norm of $M$.
The matrix $I$ is the identity matrix, whose dimension should be clear from context.

\begin{definition}[Second-order stationary point]
\label{def:second-order-stationary}
A $(\tau_g, \tau_h)$-stationary point of $f$ is a point $x$ so that $\lVert \nabla f(x) \rVert \leq \tau_g$ and $\lambda_\mathrm{min}(\nabla^2 f(x)) \geq -\tau_h$, where $\tau_g, \tau_h > 0$. 
\end{definition}
As is standard (e.g.~\citet{nesterov2006cubic}),
we will discuss only $(\tau, \sqrt{\rho \tau})$-stationary points, where $\rho$ is the Lipschitz constant of the Hessian. 


\section{The RMSProp Preconditioner}

Recall that methods like Adam and RMSProp replace the running sum $\sum_{i=1}^t g_i g_i^T$ used in Adagrad with an exponential moving average (EMA) of the form $(1-\beta) \sum_{i=1}^t \beta^{t-i} g_i g_i^T$, e.g. full-matrix RMSProp is described formally in Algorithm~\ref{alg:full-matrix-rmsprop}.  One key observation is that $\hat G_t = (1-\beta) \sum_{i=1}^t \beta^{t-i} g_i g_i^T \approx \E[g_t g_t^T] =: G_t$ if $\beta$ is chosen appropriately; in other words, at time $t$, the accumulated $\hat G_t$ can be seen as an approximation of the true second moment matrix $G_t = \E[g_t g_t^T]$ at the current iterate. Thus, RMSProp can be viewed as preconditioned SGD (Algorithm~\ref{alg:preconditioned-sgd}) with the preconditioner being $A_t = G_t^{-1/2}$. In practice, it is too expensive (or even infeasible) to compute $G_t$ exactly  since it requires summing over all training samples. Practical adaptive methods (see Algorithm~\ref{alg:full-matrix-rmsprop}) estimate this preconditioner (or a diagonal approximation) on-the-fly via an EMA.

\begin{algorithm}[tb]
    \caption{Preconditioned SGD}
    \label{alg:preconditioned-sgd}
    \begin{algorithmic}
        \State {\bfseries Input:} initial $x_0$, time $T$, stepsize $\eta$, preconditioner $A(x)$
        \For{$t = 0,\dots,T$}
            \State $g_t \leftarrow \text{stochastic gradient at $x_t$}$
            \State $A_t \leftarrow A(x_t)$ 
            \Comment{
            e.g. $A_t = \E[g_t g_t^T]^{-1/2}$}
            \State $x_{t+1} \leftarrow x_t - \eta A_t g_t$
        \EndFor
    \end{algorithmic}
\end{algorithm}

\begin{algorithm}[tb]
    \caption{Full-matrix RMSProp}
    \label{alg:full-matrix-rmsprop}
    \begin{algorithmic}
        \State {\bfseries Input:} initial $x_0$, time $T$, stepsize $\eta$, small number $\varepsilon >0$ for stability
        \For{$t = 0,\dots,T$}
            \State $g_t \leftarrow$ stochastic gradient
            \State $\hat G_t = \beta \hat G_{t-1} + (1-\beta) g_t g_t^T$ 
            \State $A_t = (\hat G_t + \varepsilon I)^{-1/2}$
            \State $x_{t+1} \leftarrow x_t - \eta A_t g_t$
        \EndFor
    \end{algorithmic}
\end{algorithm}

Before developing our formal results, we will build intuition about the behavior of adaptive methods by studying an idealized adaptive method (IAM) with perfect access to $G_t$. 
In the rest of this section, we make use of idealized RMSProp to answer some simple questions about adaptive methods that we feel have not yet been addressed satisfactorily.

\subsection{What is the purpose of the preconditioner?}

Why should preconditioning by $A = \E[gg^T]^{-1/2}$ help optimization? The original Adam paper~\citep{kingma2014adam} argues that Adam is an approximation to natural gradient descent, since if the objective $f$ is a log-likelihood, $\E[g g^T]$ approximates the Fisher information matrix, which captures curvature information in the space of distributions. 
There are multiple issues with 
comparing adaptive methods to natural gradient descent,
which we discuss in Appendix~\ref{appendix:hessian-noise-comparison}.
Instead, \citet{pmlr-v80-balles18a} argue that the primary function of adaptive methods is to equalize the stochastic gradient noise in each direction.
But it is still \emph{not} clear why or how equalized noise should help optimization.

Our IAM
abstraction makes it 
easy to explain precisely how rescaling the gradient noise helps.
Specifically, we manipulate the update rule for 
idealized RMSProp:
\begin{align}
    x_{t+1} &\leftarrow x_t - \eta A_t g_t \\
    &= x_t - \eta A_t \nabla_t - \eta \underbrace{A_t (g_t - \nabla_t)}_{=: \xi_t}
\end{align}
The 
$A_t \nabla_t$ 
term is deterministic; only 
$\xi_t$ is stochastic, with mean $\E[A_t (g_t - \nabla_t)] = A_t \E[g_t - \nabla_t] = 0$.
Take $\varepsilon=0$ and assume $G_t = \E[g_t g_t^T]$ is invertible, so that 
$\xi_t = G_t^{-1/2} (g_t - \nabla_t)$. 
Now we can be 
more 
precise about how RMSProp rescales gradient noise.
Specifically, we compute the covariance of the noise $\xi_t$:
\begin{align}
    \Cov(\xi_t) &= I - G_t^{-1/2} \nabla_t \nabla_t^T G_t^{-1/2}.
\end{align}
The key insight is: near stationary points, 
$\nabla_t$
will be small, so that the noise covariance $\Cov(\xi_t)$ is approximately the identity matrix $I$. In other words, at stationary points, the gradient noise is approximately isotropic.
This observation hints at why adaptive methods are so successful for nonconvex problems, where one of the main challenges is to escape saddle points~\citep{pmlr-v84-reddi18a}. \matt{other cites?}
Essentially all first-order approaches for escaping saddlepoints rely on adding carefully tuned isotropic noise, so that regardless of what the escape direction is, there is enough noise in that direction to escape with high probability.


\matt{todo: add transition sentence}
\subsection{\citep{reddi2018convergence} counterexample resolution}
Recently,
\citet{reddi2018convergence} provided a simple \emph{convex} stochastic counterexample on which RMSProp and Adam do not converge. 
Their reasoning is that RMSProp and Adam too quickly forget about large gradients from the past, in favor of small (but poor) gradients at the present.
In contrast, 
for RMSProp with the idealized preconditioner (Algorithm~\ref{alg:preconditioned-sgd} with $A = \E[gg^T]^{-1/2}$),
there is no issue, but the preconditioner $A$ cannot be computed in practice.
Rather, for this example, the exponential moving average estimation scheme fails to adequately estimate the preconditioner.

The counterexample is an optimization problem of the form 
\begin{align}
    \min_{x \in [-1,1]} \; F(x) = p f_1(x) + (1-p) f_2(x),
\end{align}
where the stochastic gradient oracle returns $\nabla f_1$ with probability $p$ and $\nabla f_2$ otherwise.
Let $\zeta > 0$ be ``small,'' and $C > 0$ be ``large.''
\citet{reddi2018convergence} set $p = (1+\zeta)/(C+1)$, $f_1(x) = C x$, and $f_2(x) = -x$. Overall, then, $F(x) = \zeta x$ which is minimized at $x =-1$, however \citet{reddi2018convergence} show that RMSProp has $\E[F(x_t)] \geq 0$ and so incurs suboptimality gap at least $\zeta$. In contrast, the idealized preconditioner is a function of
\begin{align*}
    \E[ g^2 ] &= p \left(\frac{\partial f_1}{\partial x}\right)^2 + (1-p) \left( \frac{\partial f_2}{\partial x} \right)^2
    = C (1 + \zeta) - \zeta
\end{align*}
which is a constant independent of $x$. 
Hence the preconditioner is constant, and, 
up to the choice of stepsize, idealized RMSProp on this problem is the same as SGD, which of course will converge.

The difficulty for practical adaptive methods (which estimate 
$\E[g^2]$
via an EMA) is that as $C$ grows, the variance of the estimate of $\E[g^2]$ grows too. Thus~\citet{reddi2018convergence} break Adam by making estimation of $\E[g^2]$ harder.

\section{Main Results: Gluing Estimation and Optimization}
The key enabling insight of this paper is to separately study the preconditioner 
and its estimation via EMA, then combine these to give proofs for practical adaptive methods. 
We will prove a formal guarantee that the EMA estimate $\hat G_t$ is close to the true $G_t$. 
By combining our estimation results with the underlying behavior of the preconditioner, we will be able to give convergence proofs for practical adaptive methods that are constructed in a novel, modular way.

Separating these two components enables more general results: we actually analyze preconditioned SGD (Algorithm~\ref{alg:preconditioned-sgd}) with oracle access to an arbitrary preconditioner $A(x)$. 
Idealized RMSProp is but one particular instance. 
Our convergence results depend only on specific properties of the preconditioner $A(x)$, with which we can recover convergence results for many RMSProp variants simply by bounding the appropriate constants.
For example, $A = (\E[g g^T]^{1/2} + \varepsilon I)^{-1}$ corresponds to full-matrix Adam with $\beta_1=0$ or RMSProp as commonly implemented.
For cleaner presentation, we instead focus on the variant $A = (\E[g g^T] + \varepsilon I)^{-1/2}$,
but our proof technique can handle either case or its diagonal approximation.


\subsection{Estimating from Moving Sequences}
\label{sec:estimation}

The above discussion about IAM is helpful for intuition, and as a base algorithm for analyzing convergence.
But it remains to understand how well the estimation procedure works, both for intuition's sake and for later use in a convergence proof. 
In this section we introduce an abstraction we name ``estimation from moving sequences.''
This abstraction will allow us to guarantee high quality estimates of the preconditioner, or, for that matter, any similarly constructed preconditioner. 
Our results will moreover make apparent how to choose the $\beta$ parameter in the exponential moving average: $\beta$ should increase with the stepsize $\eta$. Increasing $\beta$ over time has been supported both empirically~\citep{pmlr-v80-shazeer18a} as well as theoretically~\citep{pmlr-v70-mukkamala17a,zou2018sufficient,reddi2018convergence}, though to our knowledge, the precise pinning of $\beta$ to the stepsize $\eta$ is new.

Suppose there is a sequence of states $x_1, x_2, \dots, x_T \in \Xs$, e.g. the parameters of our model at each time step. 
We have access to the states $x_t$, but more importantly we know the states are not changing too fast: $\lVert x_t - x_{t-1} \rVert$ is bounded for all $t$. There is a Lipschitz function $G : \Xs \to \reals^{d\times d}$, which in our case is the second moment matrix of the stochastic gradients, but could be more general.
We would like to estimate $G(x)$ for each $x=x_t$, but we have only a noisy oracle $Y(x)$ for $G(x)$, which we assume is unbiased and has bounded variance. Our goal is, given noisy reads $Y_1,\dots,Y_T$ of $G(x_1),\dots,G(x_T)$, to estimate $G(x_T)$ at the current point $x_T$ as well as possible.

We consider estimators of the form $\sum_{t=1}^T w_t Y_t$. For example. setting $w_T=1$ and all others to zero would yield an unbiased (but high variance) estimate of $G(x_T)$. We could assign more mass to older samples $Y_t$, but this will introduce bias into the estimate. By optimizing this bias-variance tradeoff, we can get a good estimator. In particular, taking $w$ to be an exponential moving average (EMA) of $\{Y_t\}_{t=1}^T$ will prioritize more recent and relevant estimates, while placing enough weight on old estimates to reduce the variance. The tradeoff is controlled by the EMA parameter $\beta$; e.g. if the sequence $x_t$ moves slowly (the stepsize is small), we will want large $\beta$ because older iterates are still very relevant.

In adaptive methods, the underlying function $G(x)$ we want to estimate is $\E[g g^T]$ (or its diagonal $\E[g^2]$), and every stochastic gradient $g$ gives us an unbiased estimate $gg^T$ (resp. $g^2$) of $G(x)$. With this application in mind, we formalize our results in terms of matrix estimation. By combining standard matrix concentration inequalities (e.g. from~\citet{tropp2011freedman}) with bounds on how fast the sequence moves, we arrive at the following result, proved in Appendix~\ref{appendix:online-matrix-estimation}:

\begin{theorem}
\label{theorem:online-matrix-estimation}
Assume $\lVert x_t - x_{t+1} \rVert \leq \eta M$.
The function $G : \reals^d \to \reals^{d\times d}$ is matrix-valued and $L$-Lipschitz.
For shorthand we write $G_t := G(x_t)$.
The matrix sequence $\{Y_t : t = 1, 2, \dots\}$ is adapted to a filtration $\Fs_t$ and satisfies $\E[Y_t | \Fs_{t-1}] = G(x_t) = G_t$ for all $t \geq 1$.
For shorthand we write $G_t := G(x_t)$.
Additionally, we assume for each $t$ that
$\lVert Y_t - G_t \rVert \leq R$ and $\lVert \E[ (Y_t - G_t )^2 | \Fs_{t-1}] \rVert \leq \sigma_\mathrm{max}^2$.
Set $w_t \propto \beta^{T - t}$ with $\sum_{t=1}^T w_t = 1$ and assume $T > 4 / (1-\beta)$. 
Then with probability $1-\delta$, the estimation error $\Phi = \left\lVert \sum_{t=1}^T w_t Y_t - G_T \right\rVert$ is bounded by
\begin{align*}
    \Phi \leq O(\sigma_\mathrm{max} \sqrt{1-\beta} \sqrt{\log(d/\delta)} + M L \eta / (1-\beta)).
\end{align*}
This is optimized by $\beta = 1 - C \eta^{2/3}$, for which the bound is $O( (\eta M \sigma_\mathrm{max}^{2} (\log(d/\delta))  L)^{1/3} )$ as long as $T > C' \eta^{-2/3}$.
\end{theorem}
As long as $T$ is sufficiently large, we can get a high quality estimate of $G_t = \E[g_t g_t^T]$. 
For this, it suffices to start off the underlying optimization algorithm with $W = O(\eta^{-2/3})$ burn-in iterations where our estimate is updated but the algorithm is not started. This burn-in period will not affect asymptotic runtime as long as $W = O(\eta^{-2/3}) = O(T)$. 
In our non-convex convergence results we will require $T = O(\tau^{-4})$ and $\eta = O(\tau^{2})$, so that $W = O(\tau^{-4/3})$ which is much smaller than $T$.
In practice, one can get away with much shorter (or no) burn-in period.

If $\beta$ is properly tuned, while running an adaptive method like RMSProp, we will get good estimates of $G = \E[g g^T]$ from samples $g g^T$. However, we actually require a good estimate of $A = \E[g g^T]^{-1/2}$ and variants. To treat estimation in a unified way, we introduce estimable matrix sequences:


\begin{definition}
\label{def:estimable-sequence}
A \emph{$(W, T, \eta, \esttol, \delta)$-estimable matrix sequence} is a sequence of matrices $\{A(x_t)\}_{t=1}^{W+T}$ generated from $\{x_t\}_t$ with $\lVert x_t - x_{t-1}\rVert \leq \eta$ so that with probability $1-\delta$, after a burn-in of time $W$, we can achieve an estimate sequence $\{\hat A_t\}$ so that $\lVert \hat A_t - A_t \rVert \leq \esttol$ simultaneously for all times $t=W+1,\dots,W+T$.
\end{definition}

Applying Theorem~\ref{theorem:online-matrix-estimation} and union bounding over all time $t=W+1,\dots,W+T$, we may state a concise result in terms of Definition~\ref{def:estimable-sequence}:

\begin{proposition}
Suppose $G = \E[g_t g_t^T]$ is $L$-Lipschitz as a function of $x$.
When applied to a generator sequence $\{x_t\}$ with $\lVert x_t - x_{t-1}\rVert \leq \eta M$ and samples $Y_t = g_t g_t^T$, the matrix sequence $G_t = \E[g_t g_t^T]$ is $(W,T,\eta M,\esttol,\delta)$-estimable with $W = O(\eta^{-2/3})$, $T = \Omega(W)$, and $\esttol = O(\eta^{1/3} \sigma_\mathrm{max}^{2/3} (\log(2 T d / \delta)^{1/3} M^{1/3} L^{1/3})$.
\end{proposition}

We are hence guaranteed a good estimate of $G$.
What we actually want, though, is a good estimate of the preconditioner $A = (G + \varepsilon I)^{-1/2}$. In Appendix~\ref{appendix:noise-estimates-preconditioner} we show how to bound the quality of an estimate of $A$. One simple result is: 

\begin{proposition}
\label{prop:estimable-sequence-rmpsprop}
Suppose $G = \E[g g^T]$ is $L$-Lipschitz as a function of $x$.
Further suppose a uniform bound $\lambda_\mathrm{min}(G) I \preceq G(x)$ for all $x$, with $\lambda_\mathrm{min}(G) > 0$.
When applied to a generator sequence $\{x_t\}$ with $\lVert x_t - x_{t-1}\rVert \leq \eta M$ and samples $Y_t = g_t g_t^T$, the matrix sequence $A_t = (G_t + \varepsilon I)^{-1/2}$ is $(W,T,\eta M,\esttol,\delta)$-estimable with $W = O(\eta^{-2/3})$, $T = \Omega(W)$, and 
$\esttol = O( ( \eta \sigma_\mathrm{max}^{2} \log(2 T d / \delta) M L )^{1/3} (\varepsilon + \lambda_\mathrm{min}(G))^{-3/2}  )$.
\end{proposition}




\subsection{Convergence Results}
We saw in the last two sections \matt{subsections} that it is simple to reason about adaptive methods via IAM, and that it is possible to compute a good estimate of the preconditioner. But we still need to glue the two together in order to get a convergence proof for practical adaptive methods.

In this section we will give non-convex convergence results, first for IAM and then for practical realizations thereof. We start with first-order convergence as a warm-up, and then move on to second-order convergence. In each case we give a bound for IAM, study it, and then give the corresponding bound for practical adaptive methods.

\subsubsection{Assumptions and notation}
We want results for a wide variety of 
preconditioners $A$, e.g. $A = I$, the RMSProp preconditioner $A = (G + \varepsilon I)^{-1/2}$, and the diagonal version thereof, $A = (\diag(G) + \varepsilon I)^{-1/2}$.
To facilitate this and the future 
extension of our approach to other preconditioners, we give guarantees that hold for general preconditioners $A$. 
Our bounds depend on $A$ via the following properties:
\begin{definition}
\label{def:preconditioner-constants}
We say $A(x)$ is a \emph{$(\eigone, \eigtwo, \noiseub, \noiselb, \lambda_-)$-preconditioner} if, for all $x$, 
the following bounds hold. First, $\lVert A \nabla f \rVert^2 \leq \eigone \lVert A^{1/2} \nabla f \rVert^2$. Second, if $\tilde f(x)$ is the quadratic approximation of $f$ at some point $x_0$, we assume $\lVert A (\nabla f - \nabla \tilde f) \rVert \leq \eigtwo \lVert \nabla f - \nabla \tilde f \rVert$. Third, $\noiseub \geq \E[\lVert A g \rVert^2]$. Fourth, $\noiselb \leq \lambda_\mathrm{min}(A \E[g g^T] A^T)$. Finally, $\lambda_- \leq \lambda_\mathrm{min}(A)$.
\end{definition}
Note that we could bound $\eigone = \eigtwo = \lambda_\mathrm{max}(A)$.
but in practice $\eigone$ and $\eigtwo$ may be smaller, since they depend on the behavior of $A$ only in specific directions. 
In particular, if the preconditioner $A$ is well-aligned with the Hessian, as may be the case if the natural gradient approximation is valid, then $\eigone$ would be very small. 
If $f$ is exactly quadratic, $\eigtwo$ can be taken as a constant.
The constant $\noiseub$ controls the magnitude of (rescaled) gradient noise, which affects stability at a local minimum.
Finally, $\noiselb$ gives a lower bound on the amount of gradient noise in any direction; 
when $\noiselb$ is larger it is easier to escape saddle points\footnote{In cases where $G = \E[g g^T]$ is rank deficient, e.g. in high-dimensional finite sum problems, lower bounds on $\lambda_\mathrm{min}(G)$ should be understood as lower bounds on $\E[ (v^T g)^2]$ for escape directions $v$ from saddle points, analogous to the ``CNC condition'' from \citep{pmlr-v80-daneshmand18a}.}.
For shorthand, a $(\cdot,\cdot,\noiseub,\cdot,\lambda_-)$-preconditioner needs to satisfy only the corresponding inequalities.

In Appendix~\ref{appendix:constants} we provide bounds on these constants for several variants of the second moment preconditioner.
Below we highlight the two most relevant cases, corresponding to SGD and RMSProp:
\begin{proposition}
\label{prop:sgd-constants}
The preconditioner $A = I$ is a $(\eigone,\eigtwo,\noiseub,\noiselb,\lambda_-)$-preconditioner, with
$\eigone = \eigtwo = 1$, $\noiseub \leq \E[\lVert g \rVert^2] \leq d \cdot\trace(G)$, $\noiselb \leq \lambda_\mathrm{min}(G)$, and $\lambda_- = 1$.
\end{proposition}

\begin{proposition}
\label{prop:full-matrix-reg1-constants}
The preconditioner $A = (G+\varepsilon I)^{-1/2}$ is a $(\eigone,\eigtwo,\noiseub,\noiselb,\lambda_-)$-preconditioner, with
\begin{align*}
    \eigone &= \eigtwo = \frac{1}{(\lambda_\mathrm{min}(G) + \varepsilon)^{1/2}}, 
    \quad 
    \noiseub = \frac{d \lambda_\mathrm{max}(G)}{\varepsilon + \lambda_\mathrm{max}(G)}, \\
    \noiselb &= \frac{\lambda_\mathrm{min}(G)}{\lambda_\mathrm{min}(G) + \varepsilon}, \quad \text{and}\quad \lambda_- = (\lambda_\mathrm{max}(G) + \varepsilon)^{-1/2}.
\end{align*}
\end{proposition}




\subsubsection{First-order convergence}
Proofs are given in Appendix~\ref{appendix:first-order}.
For all first-order results, we assume that $A$ is a $(\cdot,\cdot,\noiseub,\cdot,\lambda_-)$-preconditioner. 
The proof technique is essentially standard, with minor changes in order to accomodate general preconditioners.
First, suppose we have exact oracle access to the preconditioner:
\begin{theorem}
\label{thm:first-order-no-error}
Run preconditioned SGD with preconditioner $A$ and stepsize $\eta = \tau^2 \lambda_- / (L \noiseub)$. For small enough $\tau$, after $T = 2 (f(x_0) - f^*) L \noiseub / (\tau^4 \lambda_-^2)$ iterations, 
\begin{align}
    \frac{1}{T} \sum_{t=0}^{T-1} \E\left[ \lVert \nabla f(x_t) \rVert^2 \right] \leq \tau^2.
\end{align}
\end{theorem}

Now we consider an alternate version where instead of the preconditioner $A_t$, we precondition by an noisy version $\hat A_t$ that is close to $A_t$, i.e. $\lVert \hat A_t - A_t \rVert \leq \esttol$.

\begin{theorem}
\label{thm:first-order-with-error}
Suppose we have access to an inexact preconditioner $\hat A$, which satisfies $\lVert \hat A - A \lVert \leq \esttol$
for $\esttol < \lambda_- / 2$.
Run preconditioned SGD with preconditioner $\hat A$ and stepsize $\eta = \tau^2 \lambda_- / (4 \sqrt2 L \noiseub)$.
For small enough $\tau$, after $T = 32 (f(x_0) - f^*) L \noiseub / (\tau^4 \lambda_-^2)$
iterations, we will have
\begin{align}
    \frac{1}{T} \sum_{t=0}^{T-1} \E\left[ \lVert \nabla f(x_t) \rVert^2 \right] \leq \tau^2.
\end{align}
\end{theorem}

The results are the same up to constants.
In other words, as long as we can achieve less than $\lambda_- / 2$ error, we will converge at essentially the same rate as if we had the exact preconditioner. 
In light of this, for the second-order convergence results, we treat only the noisy version.

Theorem~\ref{thm:first-order-with-error} gives a convergence bound assuming a good estimate of the preconditioner, and our estimation results
guarantee a good estimate.
By gluing together Theorem~\ref{thm:first-order-with-error} with our estimation results for the RMSProp preconditioner, i.e. 
Proposition~\ref{prop:estimable-sequence-rmpsprop}, we can give a convergence result for bona fide RMSProp:
\begin{corollary}
Consider RMSProp with burn-in, as in Algorithm~\ref{alg:preconditioned-sgd-burnin}, where we
estimate $A = (G + \varepsilon I)^{-1/2}$.
Retain the same choice of $\eta = O(\tau^2)$ and $T=O(\tau^{-4})$ as in Theorem~\ref{thm:first-order-with-error}.
For small enough $\tau$, such a choice of $\eta$ will yield $\esttol < \lambda_- / 2$.
Choose all other parameters e.g. $\beta$ in accordance with Proposition~\ref{prop:estimable-sequence-rmpsprop}.
In particular, choose $W = \Theta(\eta^{-2/3}) = \Theta(\tau^{-4/3}) = O(T)$ for the burn-in parameter. 
Then with probability $1-\delta$, in overall time $O(W + T) = O(\tau^{-4})$, we achieve 
\begin{align}
    \frac{1}{T} \sum_{t=0}^{T-1} \E\left[ \lVert \nabla f(x_t) \rVert^2 \right] \leq \tau^2.
\end{align}
\end{corollary}


\begin{algorithm}[tb]
    \caption{RMSProp with burn-in}
    \label{alg:preconditioned-sgd-burnin}
    \begin{algorithmic}
        \State {\bfseries Input:} initial $x_0$, time $T$, stepsize $\eta$, burn-in length $W$
        \State $\hat G_0 \leftarrow \textsc{BurnIn}(W, \beta)$ \Comment{Appendix~\ref{appendix:alg}}
        \For{$t = 0,\dots,T$}
            \State $g_t \leftarrow \text{stochastic gradient}$
            \State $\hat G_t \leftarrow \beta \hat G_{t-1} + (1-\beta) g_t g_t^T$
            \State $\hat A_t \leftarrow \hat G_t^{-1/2}$
            \State $x_{t+1} \leftarrow x_t - \eta \hat A_t g_t$
        \EndFor
    \end{algorithmic}
\end{algorithm}

\subsubsection{Second-order convergence}

Now we leverage the power of our high level approach to prove nonconvex second-order convergence for adaptive methods.
Like the first-order results, we start by proving convergence bounds for a generic, possibly inexact preconditioner $A$.
Our proof is based on that of~\citet{pmlr-v80-daneshmand18a}, though our study of the preconditioner is wholly new.
Accordingly, we study the convergence of Algorithm~\ref{alg:idealized-full-matrix-large-stepsize}, which is the same as Algorithm~\ref{alg:preconditioned-sgd} (generic preconditioned SGD) except that once in a while we take a large stepsize so we may escape saddlepoints.
The proof is given completely in Appendix~\ref{appendix:main-proof}. 
At a high level, we show the algorithm makes progress when the gradient is large and when we are at a saddle point, and does not escape from local minima.
Our analysis uses all the constants specified in Definition~\ref{def:preconditioner-constants}, e.g. the speed of escape from saddle points depends on $\noiselb$, the lower bound on stochastic gradient noise.

Then, as before, we simply fuse our convergence guarantees with our estimation guarantees.
The end result is, to our knowledge, the first nonconvex second-order convergence result for any adaptive method.

\paragraph{Definitions for second-order results.}
Assume further that the Hessian $H$ is $\rho$-Lipschitz and the preconditioner $A(x)$ is $\alpha$-Lipschitz.
The dependence on these constants is made more precise in the proof, in Appendix~\ref{appendix:main-proof}.
The usual stepsize is $\eta$, while $r$ is the occasional large stepsize that happens every $t_\mathrm{thresh}$ iterations.
The constant $\delta$ is the small probability of failure we tolerate.
For all results, we assume $A$ is a $(\eigone,\eigtwo,\noiseub,\noiselb,\lambda_-)$-preconditioner.
For simplicity, we assume the noisy estimate $\hat A$ also satisfies the $\eigone$ inequality.
We will also assume a uniform bound on $\lVert A g \rVert \leq M = O(\sqrt{\noiseub})$. 

The proofs rely on a few other quantities that we 
optimally determine as a function of the problem parameters:
$f_\mathrm{thresh}$ is a threshold on the function value progress, and
$g_\mathrm{thresh} = f_\mathrm{thresh} / t_\mathrm{thresh}$ is the time-amortized average of $f_\mathrm{thresh}$.
We specify the precise values of all quantities in the proof.


\begin{algorithm}[tb]
    \caption{Preconditioned SGD with increasing stepsize}
    \label{alg:idealized-full-matrix-large-stepsize}
    \begin{algorithmic}
        \State {\bfseries Input:} initial $x_0$, time $T$, stepsizes $\eta, r$, threshold $t_\mathrm{thresh}$, matrix error $\esttol$
        \For{$t = 0,\dots,T$}
            \State $A_t \leftarrow A(x_t)$ \Comment{preconditioner at $x_t$}
            \State $\hat A_t \leftarrow \text{any matrix with $\lVert \hat A_t - A_t \rVert \leq \esttol$}$
            \State $g_t \leftarrow \text{stochastic gradient at } x_t$
            \If{$t \text{ mod } t_\mathrm{thresh} = 0$}
                \State $x_{t+1} \leftarrow x_t - r \hat A_t g_t$ 
            \Else
                \State $x_{t+1} \leftarrow x_t - \eta \hat A_t g_t$
            \EndIf
        \EndFor
    \end{algorithmic}
\end{algorithm}

\begin{theorem}
\label{thm:second-order-with-error}
Consider Algorithm~\ref{alg:idealized-full-matrix-large-stepsize} 
with inexact preconditioner $\hat A_t$ and exact preconditioner $A_t$ satisfying the preceding requirements. Suppose that for all $t$, we have $\lVert \hat A_t - A_t \rVert = O(\tau^{1/2})$. 
Then for small 
$\tau$,
with probability $1-\delta$, we reach an $(\tau, \sqrt{\rho \tau})$-stationary point in time 
\begin{align}
    T = \tilde O\left( \frac{\eigone^4 \eigtwo^4 \noiseub^4}{\lambda_-^{10} \noiselb^4} \cdot \frac{L^3}{\rho \delta^3} \cdot \tau^{-5} \right).
\end{align}
The big-O suppresses other constants 
given in the proof.
\end{theorem}




To prove a result for bona fide RMSProp, we 
need to combine Theorem~\ref{thm:second-order-with-error} with an algorithm that maintains a good estimate of $G = \E[g g^T]$ (and consequently $A = (G + \varepsilon I)^{-1/2}$). 
This is 
more delicate than the first-order case because now
the stepsize varies.
Whenever we take a large stepsize, the estimation algorithm will need to hallucinate $S$ number of smaller steps in order to keep the estimate accurate.
Our overall scheme is formalized in Appendix~\ref{appendix:alg}, for which the following convergence result holds:
\begin{corollary}
Consider the RMSProp version of Algorithm~\ref{alg:idealized-full-matrix-large-stepsize} that is described in Appendix~\ref{appendix:alg}. 
Retain the same choice of $\eta = O(\tau^{5/2})$, $r=O(\tau)$, and $T=O(\tau^{-5})$ as in Theorem~\ref{thm:second-order-with-error}.
For small enough $\tau$, such a choice of $\eta$ will yield $\esttol < \lambda_- / 2$.
Choose $W = \Theta(\eta^{-2/3}) = \Theta(\tau^{-5/3}) = O(T)$ for the burn-in parameter
Choose $S = O(\tau^{-3/2})$, so that as far as the estimation scheme is concerned, the stepsize is bounded by $\max\{\eta, r / S\} = O(\tau^{5/2}) = O(\eta)$. 
Then as before, with probability $1-\delta$, we can reach an $(\tau, \sqrt{\rho \tau})$-stationary point in total time 
\begin{align}
    W + T = \tilde O\left( \frac{\eigone^4 \eigtwo^4 \noiseub^4}{\lambda_-^{10} \noiselb^4} \cdot \frac{L^3}{\rho \delta^3} \cdot \tau^{-5} \right),
\end{align}
where $\eigone,\eigtwo,\noiseub,\noiselb,\lambda_-$ are the constants describing $A = (G + \varepsilon I)^{-1/2}$.
\end{corollary}
Again, as in the first order results, one could substitute in any other estimable preconditioner, such as the more common diagonal version $A = (\diag(G) + \varepsilon I )^{-1/2}$.


\section{Discussion}
\label{sec:discussion}

Separating the estimation step from the preconditioning enables evaluation of different choices for the preconditioner. 

\subsection{How to set the regularization parameter $\varepsilon$}
In the adaptive methods literature, it is still a mystery how to properly set the regularization parameter $\varepsilon$ that ensures invertibility of $G + \varepsilon I$. 
When the optimality tolerance $\tau$ is small enough, estimating the preconditioner is not the bottleneck. Thus, focusing only on the idealized case, one could just choose $\varepsilon$ to minimize the bound.
Our first-order results depend on $\varepsilon$ only through the following term:
\begin{align}
    \frac{\noiseub}{\lambda_\mathrm{min}(A)} \leq \frac{d \lambda_\mathrm{min}(G)}{\varepsilon + \lambda_\mathrm{min}(G)} \cdot (\lambda_\mathrm{max}(G) + \varepsilon),
\end{align}
where we have used the preconditioner bounds from Proposition~\ref{prop:full-matrix-reg1-constants}.
This is minimized by taking $\varepsilon \to \infty$, which  
suggests using identity preconditioner, or SGD.
In contrast, for second-order convergence, the bound is 
\begin{align}
    \frac{\eigone^4 \eigtwo^4 \noiseub^4}{\lambda_-^{10} \noiselb^4} 
    &\leq d^4 \kappa(G)^4 (\lambda_\mathrm{max}(G) + \varepsilon),
\end{align}
which is instead minimized with $\varepsilon = 0$.
So for the best second-order convergence rate, it is desireable to set $\varepsilon$ as small as possible. 
Note that since our bounds hold only for small enough convergence tolerance $\tau$, it is possible that the optimal $\varepsilon$ should depend in some way on $\tau$. 

\subsection{Comparison to SGD}
Another important question we make progress towards is: when are adaptive methods better than SGD?
Our second-order result depends on the preconditioner only through $\eigone^4 \eigtwo^4 \noiseub^4 / (\lambda_-^{10} \noiselb^4)$.
Plugging in Proposition~\ref{prop:sgd-constants} for SGD, we may bound 
\begin{align}
	\frac{\eigone^4 \eigtwo^4 \noiseub^4}{\lambda_-^{10} \noiselb^4} \leq \frac{\E[\lVert g \rVert^2]^4}{\lambda_\mathrm{min}(G)^4}
    \leq d^4 \kappa(G)^4,
\end{align}
while for full-matrix RMSProp, we have
\begin{align}
    \frac{\eigone^4 \eigtwo^4 \noiseub^4}{\lambda_-^{10} \noiselb^4} 
    &\leq d^4 \kappa(G)^4 (\lambda_\mathrm{max}(G) + \varepsilon).
\end{align}
Setting $\varepsilon = 0$ for simplicity, we conclude that full-matrix RMSProp converges faster if $\lambda_\mathrm{max}(G) \leq 1$.

Now suppose that for a given optimization problem, the preconditioner $A$ is well-aligned with the Hessian so that $\eigone = O(1)$ (e.g. if the natural gradient approximation holds) and that near saddle points the objective is essentially quadratic so that $\eigtwo = O(1)$. In this regime, 
the preconditioner dependence of idealized full matrix 
RMSProp is $d^4 \lambda_\mathrm{max}(G)^5$, which yields a better result than SGD when $\lambda_\mathrm{max}(G) \leq \lambda_\mathrm{min}(G)^{-4}$.
This will happen whenever $\lambda_\mathrm{min}(G)$ is relatively small. Thus, when there is not much noise in the escape direction, and the Hessian and $G^{-1/2}$ are not poorly aligned, RMSProp will converge faster overall.

\subsection{Alternative preconditioners}
Our analysis inspires the design of other preconditioners: e.g., if at each iteration we sample two independent stochastic gradients $g_1$ and $g_2$, we have unbiased sample access to $(g_1 - g_2) (g_1 - g_2)^T$, which in expectation yields the covariance $\Sigma = \Cov(g)$ instead of the second moment matrix of $g$. It immediately follows that we can prove second-order convergence results for an algorithm that constructs an exponential moving average estimate of $\Sigma$ and preconditions by $\Sigma^{-1/2}$, as advocated by~\citet{ida2017adaptive}.

\subsection{Tuning the EMA parameter $\beta$}
Another mystery of adaptive methods is how to set the exponential moving average (EMA) parameter $\beta$.
In practice $\beta$ is typically set to a constant, e.g. 0.99, while other parameters such as the stepsize $\eta$ are tuned more carefully and may vary over time.
While our estimation guarantee Theorem~\ref{theorem:online-matrix-estimation}, suggests setting $\beta = 1 - O(\eta^{2/3})$, 
the specific formula depends on constants that may be unknown, e.g. Lipschitz constants and gradient norms. Instead, one could set $\beta = 1 - C \eta^{2/3}$, and search for a good choice of the hyperparameter $C$.
For example, the common initial choice of $\eta = 0.001$ and $\beta = 0.99$ corresponds to $C = 1$.



\section{Experiments}
\label{sec:experiments}
We experimentally test our claims about adaptive methods escaping saddle points, and our suggestion for setting $\beta$.

\paragraph{Escaping saddle points.}
First, we test our claim that when the gradient noise is ill-conditioned, adaptive methods escape saddle points faster than SGD, and often converge faster to (approximate) local minima.
We construct a two dimensional\footnote{The same phenomenon still holds in higher dimensions but the presentation is simpler with $d=2$.} 
non-convex problem $f(x) = \frac1n \sum_{i=1}^n f_i(x)$ where $f_i(x) = \frac12 x^T H x + b_i^T x + \lVert x \rVert_{10}^{10}$. 
Here, $H = \diag([1, -0.1])$, so $f$ has a saddle point at the origin with objective value zero. The vectors $b_i$ are chosen so that sampling $b$ uniformly from $\{b_i\}_{i=1}^n$ yields $\E[b]=0$ and $\Cov(b) = \diag([1, 0.01])$. Hence at the origin there is an escape direction but little gradient noise in that direction.

We initialize SGD and (diagonal) RMSProp (with $\beta = 1 - \eta^{2/3}$) at the saddle point and test several stepsizes $\eta$ for each. 
Results for the first $10^4$ iterations are shown in Figure~\ref{fig:saddlepoint-compare}.
In order to escape the saddle point as fast as RMSProp, SGD requires a substantially larger stepsize,
e.g. SGD needs $\eta = 0.01$ to escape as fast as RMSProp does with $\eta = 0.001$.
But with such a large stepsize, SGD cannot converge to a small neighborhood of the local minimum, and instead bounces around due to gradient noise.
Since RMSProp can escape with a small stepsize, it can converge to a much smaller neighborhood of the local minimum.
Overall, for any fixed final convergence criterion, RMSProp escapes faster and converges faster overall.

\begin{figure}
  \centering
  \includegraphics[width=3.2in]{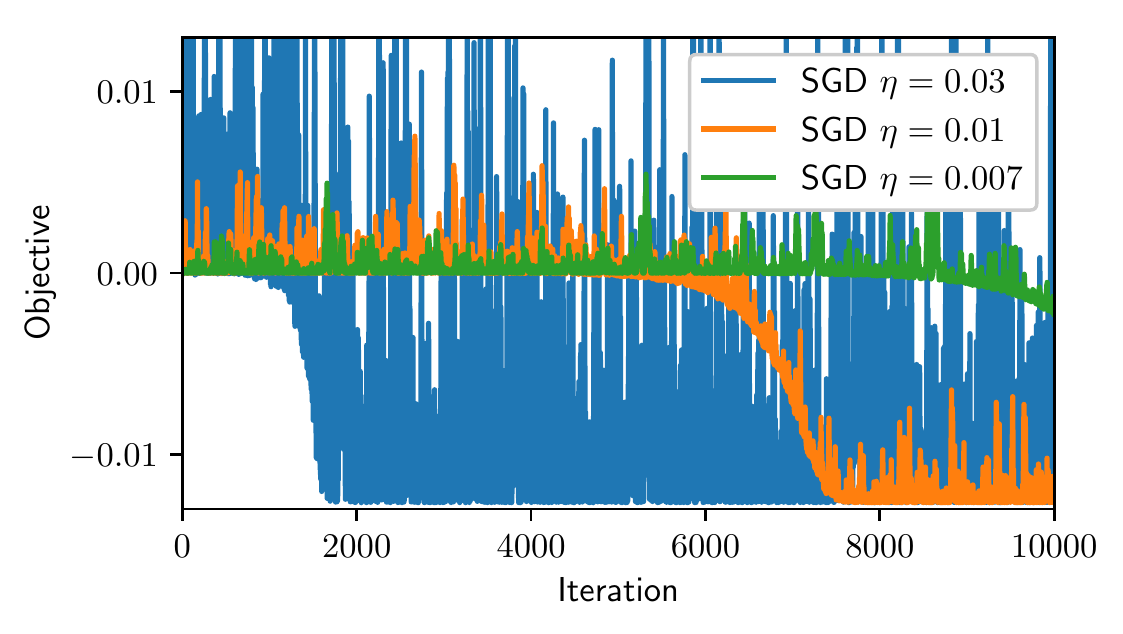}
  \includegraphics[width=3.2in]{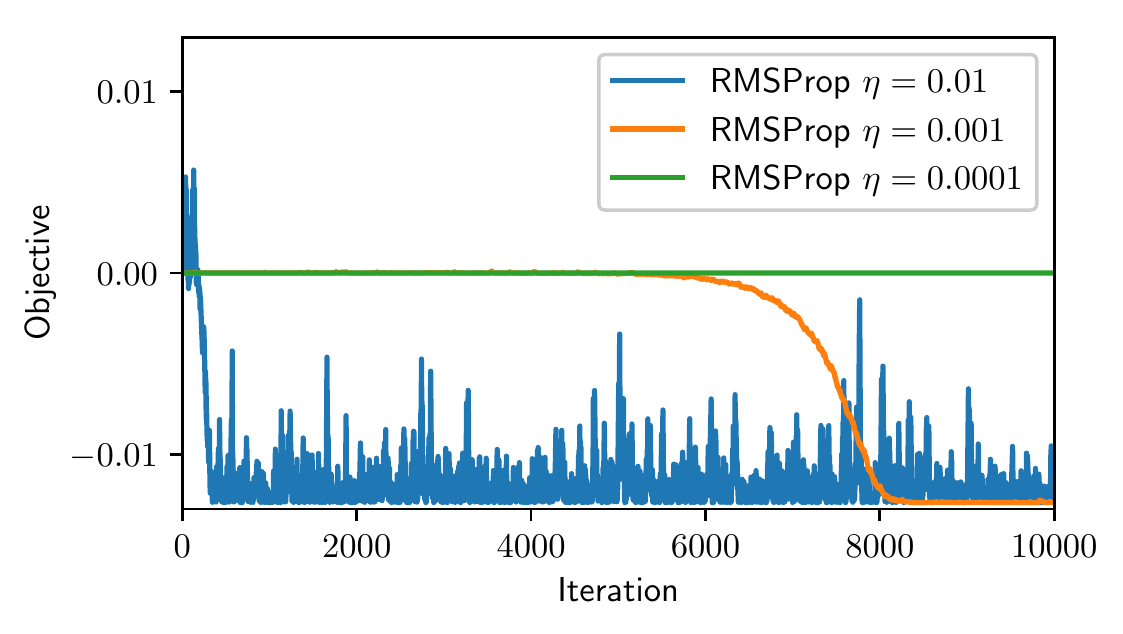}
  \vspace{-0.125in}
  \caption{SGD (left) vs RMSProp (right) performance escaping a saddle point with poorly conditioned gradient noise. Compared to RMSProp, SGD requires a larger stepsize to escape as quickly, which negatively impacts convergence to the local minimum.}
  \label{fig:saddlepoint-compare}
\end{figure}

\paragraph{Setting the EMA parameter $\beta$.}
Next, we test our recommendations regarding setting the EMA parameter $\beta$. 
We consider logistic regression on MNIST.
We use (diagonal) RMSProp with batch size 100, decreasing stepsize $\eta_t = 0.001 / \sqrt{t}$ and $\varepsilon = 10^{-8}$, and compare different schedules for $\beta$.
Specifically we test $\beta \in \{0.7, 0.9, 0.97, 0.99\}$ (so that $1-\beta$ is spaced roughly logarithmically) as well as our recommendation of $\beta_t = 1 - C \eta_t^{2/3}$ for $C \in \{0.1, 0.3, 1\}$.
As shown in Figure~\ref{fig:beta-compare}, all options for $\beta$ have similar performance initially, but as $\eta_t$ decreases, large $\beta$ yields substantially better performance. In particular, our decreasing $\beta$ schedule achieved the best performance, and moreover was insensitive to how $C$ was set.

\begin{figure}[t]
  \centering
  \includegraphics[width=3.2in]{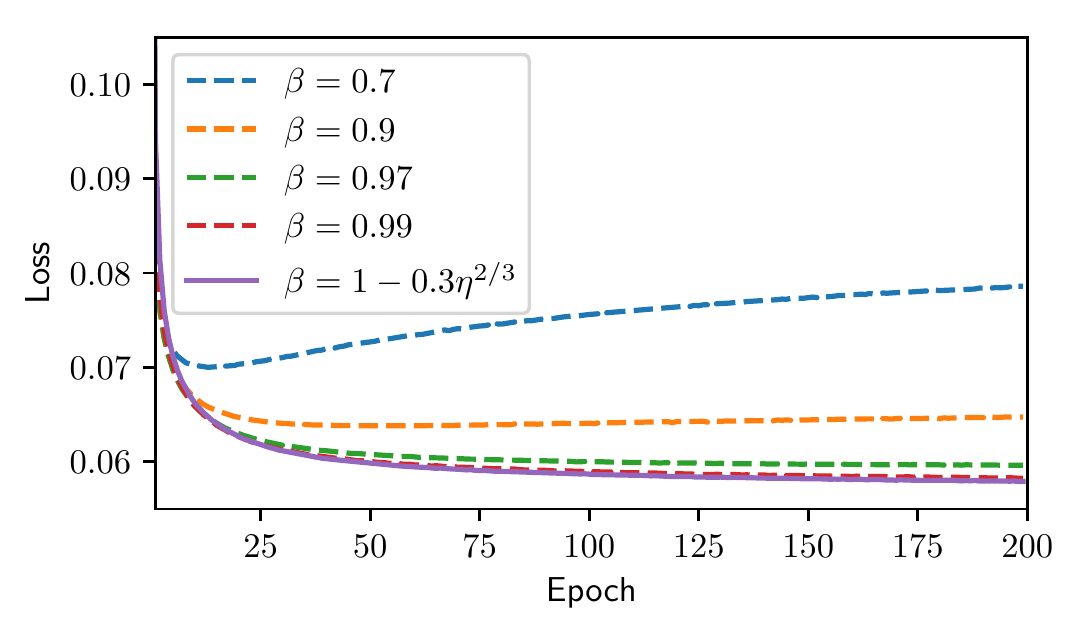}
  \vspace{-0.125in}
  \caption{Performance on MNIST logistic regression of RMSProp with different choices of $\beta$ and decreasing stepsize.}
  \label{fig:beta-compare}
\end{figure}

\subsubsection*{Acknowledgements}
This work was supported in part by the DARPA Lagrange grant, and an Amazon Research Award. We thank Nicolas Le Roux for helpful conversations.

\bibliographystyle{plainnat}
\bibliography{ref}

\begin{thebibliography}{32}
\providecommand{\natexlab}[1]{#1}
\providecommand{\url}[1]{\texttt{#1}}
\expandafter\ifx\csname urlstyle\endcsname\relax
  \providecommand{\doi}[1]{doi: #1}\else
  \providecommand{\doi}{doi: \begingroup \urlstyle{rm}\Url}\fi

\bibitem[Agarwal et~al.(2017)Agarwal, Allen-Zhu, Bullins, Hazan, and
  Ma]{Agarwal:2017:FAL:3055399.3055464}
Naman Agarwal, Zeyuan Allen-Zhu, Brian Bullins, Elad Hazan, and Tengyu Ma.
\newblock Finding approximate local minima faster than gradient descent.
\newblock In \emph{Proceedings of the 49th Annual ACM SIGACT Symposium on
  Theory of Computing}, STOC 2017, pages 1195--1199, New York, NY, USA, 2017.
  ACM.
\newblock ISBN 978-1-4503-4528-6.
\newblock \doi{10.1145/3055399.3055464}.
\newblock URL \url{http://doi.acm.org/10.1145/3055399.3055464}.

\bibitem[Agarwal et~al.(2018)Agarwal, Bullins, Chen, Hazan, Singh, Zhang, and
  Zhang]{agarwal2018case}
Naman Agarwal, Brian Bullins, Xinyi Chen, Elad Hazan, Karan Singh, Cyril Zhang,
  and Yi~Zhang.
\newblock The case for full-matrix adaptive regularization.
\newblock \emph{arXiv preprint arXiv:1806.02958}, 2018.

\bibitem[Allen-Zhu and Li(2018)]{NIPS2018_7629}
Zeyuan Allen-Zhu and Yuanzhi Li.
\newblock Neon2: Finding local minima via first-order oracles.
\newblock In S.~Bengio, H.~Wallach, H.~Larochelle, K.~Grauman, N.~Cesa-Bianchi,
  and R.~Garnett, editors, \emph{Advances in Neural Information Processing
  Systems 31}, pages 3720--3730. Curran Associates, Inc., 2018.
\newblock URL
  \url{http://papers.nips.cc/paper/7629-neon2-finding-local-minima-via-first-order-oracles.pdf}.

\bibitem[Balles and Hennig(2018)]{pmlr-v80-balles18a}
Lukas Balles and Philipp Hennig.
\newblock Dissecting {Adam}: The sign, magnitude and variance of stochastic
  gradients.
\newblock In Jennifer Dy and Andreas Krause, editors, \emph{Proceedings of the
  35th International Conference on Machine Learning}, volume~80 of
  \emph{Proceedings of Machine Learning Research}, pages 404--413,
  Stockholmsmässan, Stockholm Sweden, 10--15 Jul 2018. PMLR.
\newblock URL \url{http://proceedings.mlr.press/v80/balles18a.html}.

\bibitem[Carmon et~al.(2018)Carmon, Duchi, Hinder, and
  Sidford]{doi:10.1137/17M1114296}
Y.~Carmon, J.~Duchi, O.~Hinder, and A.~Sidford.
\newblock Accelerated methods for nonconvex optimization.
\newblock \emph{SIAM Journal on Optimization}, 28\penalty0 (2):\penalty0
  1751--1772, 2018.
\newblock \doi{10.1137/17M1114296}.
\newblock URL \url{https://doi.org/10.1137/17M1114296}.

\bibitem[Chen et~al.(2018)Chen, Liu, Sun, and Hong]{chen2018convergence}
Xiangyi Chen, Sijia Liu, Ruoyu Sun, and Mingyi Hong.
\newblock On the convergence of a class of {adam}-type algorithms for
  non-convex optimization.
\newblock \emph{arXiv preprint arXiv:1808.02941}, 2018.

\bibitem[Daneshmand et~al.(2018)Daneshmand, Kohler, Lucchi, and
  Hofmann]{pmlr-v80-daneshmand18a}
Hadi Daneshmand, Jonas Kohler, Aurelien Lucchi, and Thomas Hofmann.
\newblock Escaping saddles with stochastic gradients.
\newblock In Jennifer Dy and Andreas Krause, editors, \emph{Proceedings of the
  35th International Conference on Machine Learning}, volume~80 of
  \emph{Proceedings of Machine Learning Research}, pages 1155--1164,
  Stockholmsmässan, Stockholm Sweden, 10--15 Jul 2018. PMLR.
\newblock URL \url{http://proceedings.mlr.press/v80/daneshmand18a.html}.

\bibitem[Duchi et~al.(2011)Duchi, Hazan, and
  Singer]{Duchi:2011:ASM:1953048.2021068}
John Duchi, Elad Hazan, and Yoram Singer.
\newblock Adaptive subgradient methods for online learning and stochastic
  optimization.
\newblock \emph{J. Mach. Learn. Res.}, 12:\penalty0 2121--2159, July 2011.
\newblock ISSN 1532-4435.
\newblock URL \url{http://dl.acm.org/citation.cfm?id=1953048.2021068}.

\bibitem[Ge et~al.(2015)Ge, Huang, Jin, and Yuan]{pmlr-v40-Ge15}
Rong Ge, Furong Huang, Chi Jin, and Yang Yuan.
\newblock Escaping from saddle points --- online stochastic gradient for tensor
  decomposition.
\newblock In Peter Grünwald, Elad Hazan, and Satyen Kale, editors,
  \emph{Proceedings of The 28th Conference on Learning Theory}, volume~40 of
  \emph{Proceedings of Machine Learning Research}, pages 797--842, Paris,
  France, 03--06 Jul 2015. PMLR.
\newblock URL \url{http://proceedings.mlr.press/v40/Ge15.html}.

\bibitem[Goodfellow et~al.(2016)Goodfellow, Bengio, and
  Courville]{Goodfellow-et-al-2016}
Ian Goodfellow, Yoshua Bengio, and Aaron Courville.
\newblock \emph{Deep Learning}.
\newblock MIT Press, 2016.
\newblock \url{http://www.deeplearningbook.org}.

\bibitem[Ida et~al.(2017)Ida, Fujiwara, and Iwamura]{ida2017adaptive}
Yasutoshi Ida, Yasuhiro Fujiwara, and Sotetsu Iwamura.
\newblock Adaptive learning rate via covariance matrix based preconditioning
  for deep neural networks.
\newblock In \emph{Proceedings of the Twenty-Sixth International Joint
  Conference on Artificial Intelligence, {IJCAI-17}}, pages 1923--1929, 2017.
\newblock \doi{10.24963/ijcai.2017/267}.
\newblock URL \url{https://doi.org/10.24963/ijcai.2017/267}.

\bibitem[Jacobs(1988)]{jacobs1988increased}
Robert~A Jacobs.
\newblock Increased rates of convergence through learning rate adaptation.
\newblock \emph{Neural networks}, 1\penalty0 (4):\penalty0 295--307, 1988.

\bibitem[Jin et~al.(2017)Jin, Ge, Netrapalli, Kakade, and
  Jordan]{pmlr-v70-jin17a}
Chi Jin, Rong Ge, Praneeth Netrapalli, Sham~M. Kakade, and Michael~I. Jordan.
\newblock How to escape saddle points efficiently.
\newblock In Doina Precup and Yee~Whye Teh, editors, \emph{Proceedings of the
  34th International Conference on Machine Learning}, volume~70 of
  \emph{Proceedings of Machine Learning Research}, pages 1724--1732,
  International Convention Centre, Sydney, Australia, 06--11 Aug 2017. PMLR.
\newblock URL \url{http://proceedings.mlr.press/v70/jin17a.html}.

\bibitem[Kingma and Ba(2014)]{kingma2014adam}
Diederik~P Kingma and Jimmy Ba.
\newblock Adam: A method for stochastic optimization.
\newblock \emph{arXiv preprint arXiv:1412.6980}, 2014.

\bibitem[Lee et~al.(2016)Lee, Simchowitz, Jordan, and Recht]{pmlr-v49-lee16}
Jason~D. Lee, Max Simchowitz, Michael~I. Jordan, and Benjamin Recht.
\newblock Gradient descent only converges to minimizers.
\newblock In Vitaly Feldman, Alexander Rakhlin, and Ohad Shamir, editors,
  \emph{29th Annual Conference on Learning Theory}, volume~49 of
  \emph{Proceedings of Machine Learning Research}, pages 1246--1257, Columbia
  University, New York, New York, USA, 23--26 Jun 2016. PMLR.
\newblock URL \url{http://proceedings.mlr.press/v49/lee16.html}.

\bibitem[Levy(2016)]{levy2016power}
Kfir~Y Levy.
\newblock The power of normalization: Faster evasion of saddle points.
\newblock \emph{arXiv preprint arXiv:1611.04831}, 2016.

\bibitem[McMahan and Streeter(2010)]{DBLP:conf/colt/McMahanS10}
H.~Brendan McMahan and Matthew~J. Streeter.
\newblock Adaptive bound optimization for online convex optimization.
\newblock In \emph{{COLT} 2010 - The 23rd Conference on Learning Theory, Haifa,
  Israel, June 27-29, 2010}, pages 244--256, 2010.
\newblock URL
  \url{http://colt2010.haifa.il.ibm.com/papers/COLT2010proceedings.pdf\#page=252}.

\bibitem[Mokhtari et~al.(2018)Mokhtari, Ozdaglar, and
  Jadbabaie]{mokhtari2018escaping}
Aryan Mokhtari, Asuman Ozdaglar, and Ali Jadbabaie.
\newblock Escaping saddle points in constrained optimization.
\newblock In S.~Bengio, H.~Wallach, H.~Larochelle, K.~Grauman, N.~Cesa-Bianchi,
  and R.~Garnett, editors, \emph{Advances in Neural Information Processing
  Systems 31}, pages 3633--3643. Curran Associates, Inc., 2018.
\newblock URL
  \url{http://papers.nips.cc/paper/7621-escaping-saddle-points-in-constrained-optimization.pdf}.

\bibitem[Mukkamala and Hein(2017)]{pmlr-v70-mukkamala17a}
Mahesh~Chandra Mukkamala and Matthias Hein.
\newblock Variants of {RMSP}rop and {A}dagrad with logarithmic regret bounds.
\newblock In Doina Precup and Yee~Whye Teh, editors, \emph{Proceedings of the
  34th International Conference on Machine Learning}, volume~70 of
  \emph{Proceedings of Machine Learning Research}, pages 2545--2553,
  International Convention Centre, Sydney, Australia, 06--11 Aug 2017. PMLR.
\newblock URL \url{http://proceedings.mlr.press/v70/mukkamala17a.html}.

\bibitem[Nesterov(2013)]{nesterov2013introductory}
Yurii Nesterov.
\newblock \emph{Introductory lectures on convex optimization: A basic course},
  volume~87.
\newblock Springer Science \& Business Media, 2013.

\bibitem[Nesterov and Polyak(2006)]{nesterov2006cubic}
Yurii Nesterov and Boris~T Polyak.
\newblock Cubic regularization of {Newton} method and its global performance.
\newblock \emph{Mathematical Programming}, 108\penalty0 (1):\penalty0 177--205,
  2006.

\bibitem[Reddi et~al.(2018{\natexlab{a}})Reddi, Zaheer, Sra, Poczos, Bach,
  Salakhutdinov, and Smola]{pmlr-v84-reddi18a}
Sashank Reddi, Manzil Zaheer, Suvrit Sra, Barnabas Poczos, Francis Bach, Ruslan
  Salakhutdinov, and Alex Smola.
\newblock A generic approach for escaping saddle points.
\newblock In Amos Storkey and Fernando Perez-Cruz, editors, \emph{Proceedings
  of the Twenty-First International Conference on Artificial Intelligence and
  Statistics}, volume~84 of \emph{Proceedings of Machine Learning Research},
  pages 1233--1242, Playa Blanca, Lanzarote, Canary Islands, 09--11 Apr
  2018{\natexlab{a}}. PMLR.
\newblock URL \url{http://proceedings.mlr.press/v84/reddi18a.html}.

\bibitem[Reddi et~al.(2018{\natexlab{b}})Reddi, Kale, and
  Kumar]{reddi2018convergence}
Sashank~J. Reddi, Satyen Kale, and Sanjiv Kumar.
\newblock On the convergence of {Adam} and beyond.
\newblock In \emph{International Conference on Learning Representations},
  2018{\natexlab{b}}.
\newblock URL \url{https://openreview.net/forum?id=ryQu7f-RZ}.

\bibitem[Ruder(2016)]{ruder2016overview}
Sebastian Ruder.
\newblock An overview of gradient descent optimization algorithms.
\newblock \emph{arXiv preprint arXiv:1609.04747}, 2016.

\bibitem[Shazeer and Stern(2018)]{pmlr-v80-shazeer18a}
Noam Shazeer and Mitchell Stern.
\newblock Adafactor: Adaptive learning rates with sublinear memory cost.
\newblock In Jennifer Dy and Andreas Krause, editors, \emph{Proceedings of the
  35th International Conference on Machine Learning}, volume~80 of
  \emph{Proceedings of Machine Learning Research}, pages 4596--4604,
  Stockholmsmässan, Stockholm Sweden, 10--15 Jul 2018. PMLR.
\newblock URL \url{http://proceedings.mlr.press/v80/shazeer18a.html}.

\bibitem[Tieleman and Hinton(2012)]{tieleman2012lecture}
Tijmen Tieleman and Geoffrey Hinton.
\newblock Lecture 6.5-{RMSProp}: Divide the gradient by a running average of
  its recent magnitude.
\newblock \emph{COURSERA: Neural networks for machine learning}, 4\penalty0
  (2):\penalty0 26--31, 2012.

\bibitem[Tropp(2011)]{tropp2011freedman}
Joel Tropp.
\newblock Freedman's inequality for matrix martingales.
\newblock \emph{Electronic Communications in Probability}, 16:\penalty0
  262--270, 2011.

\bibitem[Ward et~al.(2018)Ward, Wu, and Bottou]{ward2018adagrad}
Rachel Ward, Xiaoxia Wu, and Leon Bottou.
\newblock Adagrad stepsizes: Sharp convergence over nonconvex landscapes, from
  any initialization.
\newblock \emph{arXiv preprint arXiv:1806.01811}, 2018.

\bibitem[Xu et~al.(2018)Xu, Rong, and Yang]{NIPS2018_7797}
Yi~Xu, Jing Rong, and Tianbao Yang.
\newblock First-order stochastic algorithms for escaping from saddle points in
  almost linear time.
\newblock In S.~Bengio, H.~Wallach, H.~Larochelle, K.~Grauman, N.~Cesa-Bianchi,
  and R.~Garnett, editors, \emph{Advances in Neural Information Processing
  Systems 31}, pages 5531--5541. Curran Associates, Inc., 2018.
\newblock URL
  \url{http://papers.nips.cc/paper/7797-first-order-stochastic-algorithms-for-escaping-from-saddle-points-in-almost-linear-time.pdf}.

\bibitem[Zaheer et~al.(2018)Zaheer, Reddi, Sachan, Kale, and Kumar]{zaheer}
Manzil Zaheer, Sashank Reddi, Devendra Sachan, Satyen Kale, and Sanjiv Kumar.
\newblock Adaptive methods for nonconvex optimization.
\newblock In \emph{NIPS}. 2018.

\bibitem[Zhou et~al.(2018)Zhou, Tang, Yang, Cao, and Gu]{zhou2018convergence}
Dongruo Zhou, Yiqi Tang, Ziyan Yang, Yuan Cao, and Quanquan Gu.
\newblock On the convergence of adaptive gradient methods for nonconvex
  optimization.
\newblock \emph{arXiv preprint arXiv:1808.05671}, 2018.

\bibitem[Zou et~al.(2018)Zou, Shen, Jie, Zhang, and Liu]{zou2018sufficient}
Fangyu Zou, Li~Shen, Zequn Jie, Weizhong Zhang, and Wei Liu.
\newblock A sufficient condition for convergences of {Adam} and {RMSProp}.
\newblock \emph{arXiv preprint arXiv:1811.09358}, 2018.

\end{thebibliography}

\appendix

\section{More Insights from Idealized Adaptive Methods (IAM)}
\label{appendix:hessian-noise-comparison}

\matt{this appendix section is presently the only place where we cite \citet{ruder2016overview}}

Suppose for now that we have oracle access to $G_t = \E[g_t g_t^T]$.
Why should preconditioning by $A = \E[gg^T]^{-1/2}$ help optimization? The original Adam paper~\citep{kingma2014adam} argues that Adam is an approximation to natural gradient descent, since if the objective $f$ is a log-likelihood, $\E[g g^T]$ approximates the Fisher information matrix $F$, which captures curvature information in the space of distributions. 
This connection is tenuous at best, since the approximation $F \approx \E[gg^T]$ is only valid near optimality. Moreover, the exponent is wrong: Adam preconditions by $\E[gg^T]^{-1/2}$, but natural gradient should precondition by $\E[gg^T]^{-1}$.
But using the exponent $-1$ is reported in the literature as unstable, even for Adagrad: ``\emph{without the square root operation, the algorithm performs much
worse}''~\citep{ruder2016overview}. So the exponent is changed to $-1/2$ instead of $-1$. 

Both of the above issues with the natural gradient interpretation are also pointed out in~\citet{pmlr-v80-balles18a}, who argue that the primary function of adaptive methods is to equalize the stochastic gradient noise in each direction.
But it is still \emph{not} clear precisely why or how equalized noise should help optimization.

By assuming oracle access to $\E[gg^T]$,
we can immediately 
argue that
the exponent cannot be more aggressive than $-1/2$. 
Suppose we run preconditioned SGD with the preconditioner $G_t^{-1}$ (instead of $G_t^{-1/2}$ as in RMSProp), and apply this to a noiseless problem; that is, $g_t$ always equals the full gradient $\nabla_t = \nabla f(x_t)$. The preconditioner is then 
\begin{equation}
    A_t = (\E[g_t g_t^T] + \varepsilon I)^{-1}
    = (\nabla_t \nabla_t^T + \varepsilon I)^{-1}.
\end{equation}
Taking $\varepsilon \to 0$, the
idealized RMSProp
update approaches
\begin{align}
    x_{t+1} \leftarrow x_t - \eta \frac{\nabla_t}{\lVert \nabla_t \rVert^2}.
\end{align}
First, the actual descent direction is not changed, and curvature is totally absent. 
Second, the resulting algorithm is unstable unless $\eta$ decreases rapidly: as $x_t$ approaches a stationary point, the magnitude of the step $\nabla_t / \lVert \nabla_t \rVert^2$ grows arbitrarily large, making it impossible to converge without rapidly decreasing the stepsize.

By contrast, using the standard $-1/2$ exponent and taking $\varepsilon\to0$ 
in the noiseless case 
yields normalized gradient descent:
\begin{align}
    x_{t+1} \leftarrow x_t - \eta \frac{\nabla_t}{\lVert \nabla_t \rVert}.
\end{align}
In neither case do adaptive methods actually change the direction of descent (e.g. via curvature information); only the stepsize is changed.


\section{Algorithm Details}
\label{appendix:alg}
Per our estimation results in Section~\ref{sec:estimation}, we must alter RMSProp to ensure it achieves an accurate estimate of the preconditioner. Namely, before updating the parameter $x_t$, we need to burn-in the estimate for several iterations so the initial estimate $\hat G_0$ is accurate. This subroutine is given in Algorithm~\ref{alg:burnin}.

Later, when we prove
second-order convergence,
we need to modify RMSProp to occassionally take a large step.
However, this complicates estimation: per Theorem~\ref{theorem:online-matrix-estimation}, estimation quality deteriorates as the step size increases.
Naively applying Theorem~\ref{theorem:online-matrix-estimation} to the large stepsize yields an estimate of $G$ that is not accurate enough.
To get around this, every time RMSProp takes a large step, we will hallucinate a number of smaller steps to feed into the estimation procedure. This is formalized in Algorithm~\ref{alg:hallucinate}.
Overall, the variant of RMSProp we study is formalized in Algorithm~\ref{alg:estimated-full-matrix-large-stepsize}.

\begin{algorithm}[tb]
    \caption{BurnIn}
    \label{alg:burnin}
    \begin{algorithmic}
        \Function{BurnIn}{burn-in length $W$, $\beta$}
            \For{$t = 0,\dots,W$}
                \State $g_t \leftarrow \text{stochastic gradient}$
                \State $\hat G_t \leftarrow \beta \hat G_{t-1} + (1-\beta) g_t g_t^T$
            \EndFor
            \State \Return $\hat G_t$
        \EndFunction
    \end{algorithmic}
\end{algorithm}

\begin{algorithm}[tb]
    \caption{Hallucinate}
    \label{alg:hallucinate}
    \begin{algorithmic}
        \Function{Hallucinate}{hallucination length $S$, $\beta$, $\hat G$, $x_\mathrm{start}$, $x_\mathrm{end}$}
            \For{$s = 0,\dots,S$}
                \State $g_s \leftarrow \text{stochastic gradient at $x_\mathrm{start} + \frac{s}{S} (x_\mathrm{end} - x_\mathrm{start})$}$
                \State $\hat G \leftarrow \beta \hat G + (1-\beta) g_s  g_s^T$
            \EndFor
            \State \Return $\hat G$
        \EndFunction
    \end{algorithmic}
\end{algorithm}

\begin{algorithm}[tb]
    \caption{Full-matrix RMSProp with increasing stepsize}
    \label{alg:estimated-full-matrix-large-stepsize}
    \begin{algorithmic}
        \State {\bfseries Input:} initial $x_0$, time $T$, stepsizes $\eta, r$, threshold $t_\mathrm{thresh}$, time $S$, burn-in length $W$, momentum $\beta$
        \State $\hat G_0 \leftarrow \textsc{BurnIn}(W, \beta)$ \Comment{Algorithm~\ref{alg:burnin}}

        \For{$t = 0,\dots,T$}
            \State $g_t \leftarrow \text{stochastic gradient at $x_t$}$
            \State $\hat G_t \leftarrow \beta \hat G_{t-1} + (1-\beta) g_t g_t^T$
            \State $A_t \leftarrow \hat G_t^{-1/2}$
            \If{$t \text{ mod } t_\mathrm{thresh} = 0$}
                \State $x_{t+1} \leftarrow x_t - r A_t g_t$
                \State $\hat G_t \leftarrow \textsc{Hallucinate}(S, \beta, \hat G_t, x_{t}, x_{t+1 })$ \Comment{Algorithm~\ref{alg:hallucinate}}
            \Else
                \State $x_{t+1} \leftarrow x_t - \eta A_t g_t$
            \EndIf
        \EndFor
    \end{algorithmic}
\end{algorithm}


\section{Curvature and noise constants for different preconditioners}
\label{appendix:constants}

Our analysis for general preconditioners depends on constants $\eigone,\eigtwo,\noiseub,\noiselb$, as well as $\lambda_- = \lambda_\mathrm{min}(A)$ that measure various properties of the preconditioner $A$. 
For convenience, we reproduce the definition:
\begin{definition}
We say $A(x)$ is a \emph{$(\eigone, \eigtwo, \noiseub, \noiselb, \lambda_-)$-preconditioner} if, for all $x$ in the domain, the following bounds hold. First, $\lVert A \nabla f \rVert^2 \leq \eigone \lVert A^{1/2} \nabla f \rVert^2$. Second, if $\tilde f(x)$ is the quadratic approximation of $f$ at some point $x_0$, we assume $\lVert A (\nabla f - \nabla \tilde f) \rVert \leq \eigtwo \lVert \nabla f - \nabla \tilde f \rVert$. Third, $\noiseub \geq \E[\lVert A g \rVert^2]$. Fourth, $\noiselb \leq \lambda_\mathrm{min}(A \E[g g^T] A^T)$. Finally, $\lambda_- \leq \lambda_\mathrm{min}(A)$.
\end{definition}

As before, we write $G = \E[gg^T]$ throughout.

\subsection{Constants for identity preconditioner}
In the simplest case, $A = I$ and we merely run SGD. We reproduce Proposition~\ref{prop:sgd-constants}:

\begin{proposition}
\label{prop:sgd-constants-appendix}
The preconditioner $A = I$ is a $(\eigone,\eigtwo,\noiseub,\noiselb,\lambda_-)$-preconditioner, with
$\eigone = \eigtwo = 1$, $\noiseub = \E[\lVert g \rVert^2]$, $\noiselb \leq \lambda_\mathrm{min}(G)$, and $\lambda_- = 1$.
\end{proposition}

The overall second-order complexity depends on
\begin{align}
    \frac{\eigone \eigtwo \noiseub}{\noiselb} = \frac{\E[\lVert g \rVert^2]}{\lambda_\mathrm{min}(G)},
\end{align}
as well as $\lambda_- = \lambda_\mathrm{min}(A) = 1$.

\begin{proof}[Proof of Proposition~\ref{prop:sgd-constants-appendix}]
Clearly, $\eigone = \eigtwo = \lambda_- = 1$. Then,
\begin{align}
    \E[\lVert A g \rVert^2]
    &= \E[\lVert g \rVert^2]
    =: \noiseub.
\end{align}
Finally,
\begin{align}
    \lambda_\mathrm{min}(A G A^T)
    &= \lambda_\mathrm{min}(G) =: \noiselb.
\end{align}
\end{proof}





\subsection{Constants for full matrix IAM}
Write $G = \E[g g^T]$, and define the preconditioner $A$ by $A = (G + \varepsilon I)^{-1/2}$. We reproduce Proposition~\ref{prop:full-matrix-reg1-constants}:
\begin{proposition}
\label{prop:full-matrix-reg1-constants-appendix}
The preconditioner $A = (G+\varepsilon I)^{-1/2}$ is a $(\eigone,\eigtwo,\noiseub,\noiselb,\lambda_-)$-preconditioner, with
\begin{align}
    \eigone = \eigtwo = (\lambda_\mathrm{min}(G) + \varepsilon)^{-1/2}, 
    \; \noiseub = \frac{d \lambda_\mathrm{max}(G)}{\varepsilon + \lambda_\mathrm{max}(G)}, \;  \noiselb = \frac{\lambda_\mathrm{min}(G)}{\lambda_\mathrm{min}(G) + \varepsilon},
\end{align}
and $\lambda_- = (\lambda_\mathrm{max}(G) + \varepsilon)^{-1/2}.$
\end{proposition}

Overall, the complexity depends on $\eigone \eigtwo \noiseub / \noiselb$:
\begin{align}
    \frac{\eigone \eigtwo \noiseub}{\noiselb} 
    &= \frac{1}{\sqrt{\lambda_\mathrm{min}(G) + \varepsilon}}
    \cdot \frac{1}{\sqrt{\lambda_\mathrm{min}(G) + \varepsilon}} \cdot \frac{d \lambda_\mathrm{max}(G)}{\varepsilon + \lambda_\mathrm{max}(G)} \cdot \frac{\lambda_\mathrm{min}(G) + \varepsilon}{\lambda_\mathrm{min}(G)} \\
    &= \frac{d \lambda_\mathrm{max}(G)}{(\varepsilon + \lambda_\mathrm{max}(G)) \lambda_\mathrm{min}(G)}.
\end{align}

Therefore
\begin{align}
    \frac{\eigone^4 \eigtwo^4 \noiseub^4}{\lambda_-^{10} \noiselb^4} 
    &\leq \left( \frac{d \lambda_\mathrm{max}(G)}{(\varepsilon + \lambda_\mathrm{max}(G)) \lambda_\mathrm{min}(G)} \right)^4 (\lambda_\mathrm{max}(G) + \varepsilon)^{5} \\
    &= d^4 \kappa(G)^4 (\lambda_\mathrm{max}(G) + \varepsilon)
\end{align}

Note that when $\varepsilon = 0$ and we do not regularize the preconditioner, the complexity bound is
\begin{align}
    \frac{\eigone^4 \eigtwo^4 \noiseub^4}{\lambda_-^{10} \noiselb^4} 
    &= d^4 \kappa(G)^4 \lambda_\mathrm{max}(G) .
\end{align}
If we make the optimistic but often reasonable assumptions that $\eigone = O(1)$ (if $A$ is aligned well with the Hessian) and $\eigtwo = O(1)$ (the function $f$ is essentially quadratic at saddle points) then all dependence on $\lambda_\mathrm{min}(G)$ vanishes, and the bound is
\begin{align}
    \frac{\noiseub^4}{\lambda_-^{10} \noiselb^4} 
    &= d^4 \lambda_\mathrm{max}(G)^5.
\end{align}

\begin{proof}[Proof of Proposition~\ref{prop:full-matrix-reg1-constants}]
We can bound both $\eigone$ and $\eigtwo$ by
\begin{align}
    \eigone, \eigtwo \leq \lambda_\mathrm{max}(A)
    &= \lambda_\mathrm{min}(G + \varepsilon I)^{-1/2}
    = (\lambda_\mathrm{min}(G) + \varepsilon)^{-1/2}.
\end{align}

For $\noiseub$, we need to bound $\E[ \lVert A g \rVert^2 ] = \trace( A^2 G )$. Expanding, we may write
\begin{align}
    A^2 G
    &= (G + \varepsilon I)^{-1} G.
\end{align}
The mapping $t \mapsto t / (t + \varepsilon)$ is increasing, so by using the bound $\lambda_\mathrm{max}(G) I \succeq G$, we may bound
\begin{align}
    A^2 G
    &\preceq \frac{\lambda_\mathrm{max}(G)}{\varepsilon + \lambda_\mathrm{max}(G)} I.
\end{align}
It follows that we can bound the trace of $A^2 G$ by
\begin{align}
    \noiseub = d \cdot \frac{\lambda_\mathrm{max}(G)}{\varepsilon + \lambda_\mathrm{max}(G)}.
\end{align}


Next, $\noiselb$ is a bound on the least eigenvalue of
\begin{align}
    A G A^T &= (G + \varepsilon I)^{-1/2} G (G + \varepsilon I)^{-1/2} 
    = (G + \varepsilon I)^{-1} G.
\end{align}
Since $t \mapsto t / (t + \varepsilon)$ is increasing, it is minimized when $t$ is small. Therefore
\begin{align}
    \lambda_\mathrm{min}(A G A^T) &\geq \frac{\lambda_\mathrm{min}(G)}{\lambda_\mathrm{min}(G) + \varepsilon}  =: \noiselb.
\end{align}
\end{proof}

\subsection{Constants for diagonal IAM}
Define the preconditioner $A$ by $A = \diag(\E[g^2] + \varepsilon)^{-1/2}$.

\begin{proposition}
\label{prop:diagonal-var1-constants}
The preconditioner $A = \diag(\E[g^2] + \varepsilon)^{-1/2}$ is a $(\eigone,\eigtwo,\noiseub,\noiselb,\lambda_-)$-preconditioner, with
\begin{align}
    \eigone = \eigtwo = \left(\varepsilon + \min_{i \in [d]} \E[g_i^2] \right)^{-1/2}, 
    \; \noiseub = \frac{ d \max_{i \in [d]} \E[g_i^2]}{ \varepsilon + \max_{i\in[d]} \E[g_i^2]}, \;  \noiselb = \frac{\lambda_\mathrm{min}(G \diag(G)^{-1}) \cdot \min_{i\in[d]} \E[g_i^2]}{\varepsilon + \min_{i\in[d]} \E[g_i^2]},
\end{align}
and $\lambda_- = (\varepsilon + \max_{i \in [d]} \E[g_i]^2)^{-1/2}.$
\end{proposition}

Overall,
\begin{align}
    \frac{\eigone \eigtwo \noiseub}{\noiselb} &= \frac{\varepsilon + \min_{i\in[d]} \E[g_i^2]}{ (\varepsilon + \min_{i\in[d]}\E[g_i^2]) \cdot \lambda_\mathrm{min}(G \diag(G)^{-1}) \min_{i\in[d]} \E[g_i^2]} \cdot \frac{d \cdot \max_{i\in[d]} \E[g_i^2]}{\varepsilon + \max_{i\in[d]} \E[g_i^2]} \\
    &= \frac{1}{ \lambda_\mathrm{min}(G \diag(G)^{-1}) \min_{i\in[d]} \E[g_i^2]} \cdot \frac{d \cdot \max_{i\in[d]} \E[g_i^2]}{\varepsilon + \max_{i\in[d]} \E[g_i^2]} \\
\end{align}
so the overall second-order dependence is
\begin{align}
    \frac{\eigone^4 \eigtwo^4 \noiseub^4}{\lambda_-^{10} \noiselb^4} &= \frac{(\varepsilon + \max_{i\in[d]} \E[g_i^2])^5}{ \lambda_\mathrm{min}(G \diag(G)^{-1})^4 (\min_{i\in[d]} \E[g_i^2])^4} \cdot \frac{d^4 \cdot (\max_{i\in[d]} \E[g_i^2])^4}{(\varepsilon + \max_{i\in[d])} \E[g_i^2])^4} \\
    &= \frac{(\varepsilon + \max_{i\in[d]} \E[g_i^2]) \cdot d^4 \cdot (\max_{i\in[d]} \E[g_i^2])^4}{ \lambda_\mathrm{min}(G \diag(G)^{-1})^4 (\min_{i\in[d]} \E[g_i^2])^4}.
\end{align}

If we set $\varepsilon = 0$ and do not regularize the preconditioner, the complexity bound is
\begin{align}
    \frac{\eigone^4 \eigtwo^4 \noiseub^4}{\lambda_-^{10} \noiselb^4} &= \frac{d^4 \cdot (\max_{i\in[d]} \E[g_i^2])^5}{ \lambda_\mathrm{min}(G \diag(G)^{-1})^4 (\min_{i\in[d]} \E[g_i^2])^4}.
\end{align}


\begin{proof}[Proof of Proposition~\ref{prop:diagonal-var1-constants}]
As before, we can bound both $\eigone$ and $\eigtwo$ by
\begin{align}
    \eigone, \eigtwo \leq \lambda_\mathrm{max}(A)
    &= \left( \varepsilon + \min_{i \in [d]} \E[g_i^2] \right)^{-1/2}.
\end{align}

For $\noiseub$, using the same manipulations as before, we want to bound
\begin{align}
    \E[\lVert A g \rVert^2]
    &= \trace( \diag(\E[g^2]) \diag(\varepsilon + \E[g^2])^{-1} ) \\
    &= \trace \left( \diag \left( \frac{\E[g^2]}{\varepsilon + \E[g^2]} \right) \right) \\
    &\leq d \cdot \frac{ \max_{i \in [d]} \E[g_i^2]}{ \varepsilon + \max_{i\in[d]} \E[g_i^2]}.
\end{align}

Again, bounding $\noiselb$ is difficult, as we would need to bound the least eigenvalue of
\begin{align}
    A \E[gg^T] A
    &= \E[gg^T] \diag(\varepsilon + \E[g^2])^{-1} \\
    &= G (\varepsilon + \diag(G))^{-1} \\
    &= G ( \diag(G)^{-1} - \diag(G)^{-1} (\varepsilon^{-1} I + \diag(G)^{-1})^{-1} \diag(G)^{-1} ) \\
    &= G  \diag(G)^{-1} ( I - (\varepsilon^{-1} I + \diag(G)^{-1})^{-1} \diag(G)^{-1} ).
\end{align}
The first two terms are $\noiselb$ if we had not added $\varepsilon$ to $A$. The remaining terms can be bounded as before by
\begin{align}
    I - (\varepsilon^{-1} I + \diag(G)^{-1})^{-1} \diag(G)^{-1}
    \succeq \frac{\min_{i\in[d]} \E[g_i^2]}{\varepsilon + \min_{i\in[d]} \E[g_i^2]} \cdot I
\end{align}
so that overall we can take 
\begin{align}
    \noiselb = \lambda_\mathrm{min}(G \diag(G)^{-1}) \cdot \frac{\min_{i\in[d]} \E[g_i^2]}{\varepsilon + \min_{i\in[d]} \E[g_i^2]} \leq \lambda_\mathrm{min}( G (\varepsilon + \diag(G))^{-1} ).
\end{align}

Finally,
\begin{align}
    \lambda_- = \lambda_\mathrm{min}(A) = \frac{1}{( \max_{i\in[d]}\E[g_i^2] + \varepsilon)^{1/2}}.
\end{align}
\end{proof}

\section{Main Proof}
\label{appendix:main-proof}
Here we will study the convergence of Algorithm~\ref{alg:idealized-full-matrix-large-stepsize}. This is the same as Algorithm~\ref{alg:preconditioned-sgd} except that once in a while we take a large stepsize so we may escape saddlepoints.

In order to unify our results, we prove second order convergence for general preconditioners $A(x)$. The convergence rate will depend on various properties of $A(x)$, and $A = \E[g g^T]^{-1/2}$ will turn out to be particularly well-behaved.

\subsection{Definitions}
Let $\rho$ be the Lipschitz constant of the Hessian $H$, and let $\alpha$ be the Lipschitz constant of the preconditioner matrix $A(x)$ as a function of the current iterate $x$. The usual stepsize is $\eta$, while $r$ is the occasional large stepsize that happens every $t_\mathrm{thresh}$ iterations. 
$\delta$ is a small probability of failure, $d$ is the dimension. 
Since it will recur often, we define $\kappa = (1 + \eta \gamma)$,
where $\gamma$ is the magnitude of the most negative eigenvalue of $A^{1/2} H A^{1/2}$.
By the following lemma, we will be able to lower bound $\gamma$ by $\lambda_\mathrm{min}(A) \lvert \lambda_\mathrm{min}(H) \rvert \geq \lambda_- \sqrt{\rho \tau}$:
\begin{lemma}
Suppose $A$ and $H$ are symmetric matrices, with $A \succ 0$ and $\lambda_\mathrm{min}(H) < 0$. Then there is a negative eigenvalue of $A^{1/2} H A^{1/2}$ with magnitude at least $\lambda_\mathrm{min} (A) \lvert \lambda_\mathrm{min}(H) \rvert$.
\end{lemma}
\begin{proof}
Let $v$ be the minimum eigenvector of $H$, so that $v^T H v = -\lambda_\mathrm{min}(H) \lVert v \rVert^2 = -\lambda_\mathrm{min}(H)$.
Define the unit vector $u = A^{-1/2} v / \lVert A^{-1/2} v \rVert$. Then,
\begin{align}
    u^T A^{1/2} H A^{1/2} u
    = \frac{1}{\lVert A^{-1/2} v \rVert^2} v^T H v
    = - \frac{\lambda_\mathrm{min}(H)}{\lVert A^{-1/2} v \rVert^2}.
\end{align}
The vector $u$ is not necessarily an eigenvector of $A^{1/2} H A^{1/2}$, but the above expression guarantees
that $A^{1/2} H A^{1/2}$ has a negative eigenvalue with magnitude at least
\begin{align}
    \frac{\lambda_\mathrm{min}(H)}{\lVert A^{-1/2} v \rVert^2}
    \geq \frac{\lambda_\mathrm{min}(H)}{\lambda_\mathrm{max}(A^{-1}) \lVert v \rVert^2}
    = \lambda_\mathrm{min}(H) \lambda_\mathrm{min}(A).
\end{align}
\end{proof}

Throughout, we will assume that $A$ is a $(\eigone,\eigtwo,\noiseub,\noiselb,\lambda_-)$-preconditioner, that $\hat A$ also satisfies the $\eigone$ inequality, and that $\lVert \hat A - A \rVert \leq \esttol$.

Differing from~\citet{pmlr-v80-daneshmand18a}, we will assume a uniform bound on $\lVert A g \rVert \leq M$. 
In general this bound need not depend on either the spectrum of $A$ or any uniform bound on $g$. For example, if $g$ were Gaussian, $Ag$ would be a Gaussian with zero mean and identity covariance, so we would expect $\lVert A g \rVert = O(\sqrt{d})$ with high probability.
In general $M$ should have the same scale as $\sqrt{\noiseub}$, and the statement of Theorem~\ref{thm:second-order-with-error} reflects this.

The proofs rely on a few other quantities that we will optimally determine as a function of the problem parameters:
$f_\mathrm{thresh}$ is a threshold on the function value progress, and
$g_\mathrm{thresh} = f_\mathrm{thresh} / t_\mathrm{thresh}$ is the time-amortized average of $f_\mathrm{thresh}$.

\subsection{High level picture}
For shorthand we write $A_t := A(x_t)$.
Since we want to converge to a second order stationary point, 
our overall goal is to study the event
\begin{align}
    \Es_{t} &:= \{ \lVert \nabla f(x_t) \rVert \geq \tau \text{ or } \lambda_\mathrm{min}(\nabla^2 f(x_t) ) \leq -\sqrt{\rho} \tau^{1/2} \} \\
    &= \{ \lVert \nabla f(x_t) \rVert \geq \tau \text{ or } ( \lVert \nabla f(x_t) \rVert \leq \tau \text{ and } \lambda_\mathrm{min}(\nabla^2 f(x_t) ) \leq -\sqrt{\rho} \tau^{1/2} ) \}.
\end{align}
(where $t$ is obvious from context, we will omit it.
In words, $\Es_{t}$ is the event that we are not at a second order stationary point.
The main theorem results from bounding the progress we make when $\Es_{t}$ does not yet hold, while also ensuring we do not leave once we hit a second order stationary point:
\begin{lemma}
Suppose that both
\begin{align}
    \E[f(x_{t+1}) - f(x_t) | \Es_{t} ] &\leq -g_\mathrm{thresh} \\
    \text{and } \E[f(x_{t+1}) - f(x_t) | \Es_{t}^c ] &\leq \delta g_\mathrm{thresh} / 2.
\end{align}
Set $T = 2 (f(x_0) - \min_x f(x)) / (\delta g_\mathrm{thresh})$. We return $x_t$ uniformly randomly from $x_1, \dots, x_T$. Then, with probability at least $1-\delta$, we will have chosen a time $t$ where $\Es_t$ did not occur.
\end{lemma}
\begin{proof}
Let $P_t$ be the probability that $\Es_t$ occurs. Then,
\begin{align}
    \E[f(x_{t+1}) - f(x_t)]
    &=
    \E[f(x_{t+1}) - f(x_t) | \Es_t ] P_t + \E[f(x_{t+1}) - f(x_t) | \Es_t^c ] (1-P_t) \\
    &\leq -g_\mathrm{thresh} P_t + \delta g_\mathrm{thresh} / 2 \cdot (1-P_t) \\
    &\leq  \delta g_\mathrm{thresh} / 2 -(1+\delta/2) g_\mathrm{thresh} P_t \\
    &\leq  \delta g_\mathrm{thresh} / 2 - g_\mathrm{thresh} P_t.
\end{align}
Summing over all $T$ iterations, we have:
\begin{align}
    \frac1T \sum_{t=1}^T \E[f(x_{t+1}) - f(x_t)]
    &\leq g_\mathrm{thresh} \cdot \frac1T \sum_{t=1}^T (\delta  / 2 -  P_t) \\
    \implies \frac1T \sum_{t=1}^T P_t &\leq \delta/2 + \frac{f(x_0) - f^*}{T g_\mathrm{thresh}} \leq \delta \\
    \implies \frac1T \sum_{t=1}^T (1 - P_t) &\geq 1-\delta.
\end{align}
\end{proof}

\begin{theorem}
Write $\gamma = \lambda_- \sqrt{\rho} \tau^{1/2}$.
Let $K$ be a universal constant.
The parameter $\omega$ will be set later and depends only logarithmically on the other parameters.
Set 
\begin{align*}
    r &= \gamma^2 \cdot \frac{\delta \noiselb K}{54 \eigone \eigtwo \noiseub L \rho M} \\
    \eta &= \gamma^5 \cdot \frac{\delta^2 \noiselb^2 K^2}{324 M^2 L^2 \eigone^2 \eigtwo^2 \noiseub^2 \rho^2 \omega} \\
    f_\mathrm{thresh} &= \gamma^4 \cdot \frac{\delta \noiselb^2 K^2}{54 \cdot 12 \eigone^2 \eigtwo^2 \noiseub L \rho^2 M^2 } .
\end{align*}
Let $t_\mathrm{thresh} = \omega / (\eta \gamma)$, $\esttol = O(\tau^{1/2})$, and set $g_\mathrm{thresh} = f_\mathrm{thresh} / t_\mathrm{thresh}$. 
Then we have both
\begin{align}
    \E[f(x_{t+1}) - f(x_t) | \Es_t ] &\leq -g_\mathrm{thresh} \\
    \text{and } \E[f(x_{t+1}) - f(x_t) | \Es_t^c ] &\leq \delta g_\mathrm{thresh} / 2.
\end{align}
\end{theorem}

\begin{corollary}
    In the above setting, with probability $1-\delta$, we reach an $(\tau, \sqrt{\rho} \tau^{1/2})$-stationary point in time
    \begin{align}
        \tilde O\left( \frac{M^4 L^3}{\rho \delta^3} \cdot \frac{\eigone^4 \eigtwo^4 \noiseub^2}{\lambda_-^{10} \noiselb^4} \cdot \tau^{-5} \right).
    \end{align}
\end{corollary}
\begin{proof}
Simply observe $T = C ( f_0 - f^*) / (\delta g_\mathrm{thresh})$ and plug in $g_\mathrm{thresh}$.
\end{proof}


\subsection{Amortized increase due to large stepsize iterations}
\matt{todo: figure out how to fit this in formally}
Before we start casework on whether $\Es_{t}$ holds
We want to bound the amortized effect on the objective of increasing the stepsize every $t_\mathrm{thresh}$ iterations. By Corollary~\ref{corollary:increase-not-so-bad},
\begin{align}
    \E[f(x_{t+1})] - f(x_t) \leq \frac{9 L \noiseub r^2}{8}.
\end{align}
Note that for our particular setting of $r$ and $f_\mathrm{thresh}$, we have
\begin{align}
    \frac{9 L \noiseub}{8} \cdot r^2
    &= \frac{9 L \noiseub}{8} \cdot \gamma^4 \cdot \frac{\delta^2 \noiselb^2 K^2}{54^2 \eigone^2 \eigtwo^2 \noiseub^2 L^2 \rho^2 M^2} \\
    &= \frac{9 \delta}{8 } \cdot \frac{12}{54} \cdot \gamma^4 \cdot \frac{\delta \noiselb^2 K^2}{54 \cdot 12 \eigone^2 \eigtwo^2 \noiseub L \rho^2 M^2} \\
    &= \frac{ \delta f_\mathrm{thresh}}{4},
\end{align}
so also
\begin{align}
    \E[f(x_{t+1})] - f(x_t) \leq \frac{\delta f_\mathrm{thresh}}{4}.
\end{align}
Therefore on average
\begin{align}
    \frac{\E[f(x_{t+1})] - f(x_t)}{t_\mathrm{thresh}} \leq \delta g_\mathrm{thresh} / 4.
\end{align}

\subsection{Bound on possible increase when $\Es_{t}^c$ occurs}
For the main result we need to bound
\begin{align}
    \E[ f(x_{t+1}) - f(x_t) | \Es_{t}^c] \leq \delta g_\mathrm{thresh} / 4.
\end{align}
Note that
\begin{align}
    \Es_{t}^c &= \{ \lVert \nabla f(x_t) \rVert \geq \tau \text{ or } \lambda_\mathrm{min}(\nabla^2 f(x_t)) \leq -\sqrt{\rho} \tau^{1/2} \}^c \\
    &= \{ \lVert \nabla f(x_t) \rVert < \tau \text{ and } \lambda_\mathrm{min}(\nabla^2 f(x_t)) > -\sqrt{\rho} \tau^{1/2} \}.
\end{align}
Hence it suffices to bound the function increase conditioned on $\lVert \nabla f(x_t) \rVert \leq \tau$.
By Corollary~\ref{corollary:increase-not-so-bad} we have
\begin{align}
    \E[f(x_{t+1})] - f(x_t) \leq \frac{9 L \noiseub \eta^2}{8}.
\end{align}
We want this to not exceed $\delta g_\mathrm{thresh} / 4$:
\begin{align}
    \frac{9 L \noiseub \eta^2}{8} &\stackrel{?}{\leq} \frac{\delta}{4} g_\mathrm{thresh} \\
    \Leftrightarrow \frac{9 L \noiseub \eta^2}{8} &\stackrel{?}{\leq} \frac{\delta}{4} f_\mathrm{thresh} \cdot \frac{\eta \gamma}{\omega} \\
    \Leftrightarrow \frac{9 L \noiseub \eta}{2} &\stackrel{?}{\leq} \delta f_\mathrm{thresh} \cdot \frac{\gamma}{\omega} \\
    \Leftrightarrow \frac{9 L \noiseub}{2} \cdot \gamma^5 \cdot \frac{\delta^2 \noiselb^2 K^2}{324 M^2 L^2 \eigone^2 \eigtwo^2 \noiseub^2 \rho^2 \omega}  &\stackrel{?}{\leq} \delta  \cdot \frac{\gamma}{\omega} \cdot \gamma^4 \cdot \frac{\delta \noiselb^2 K^2}{54 \cdot 12 \eigone^2 \eigtwo^2 \noiseub L \rho^2 M^2 }.
\end{align}
Cancelling like terms, we find that the inequality is equivalent to $\omega \geq 9/4$, which we can easily enforce later. Therefore we may indeed write that 
\begin{align}
    \E[f(x_{t+1})] - f(x_t) \leq \frac{\delta g_\mathrm{thresh}}{4}.
\end{align}

\subsection{Bound on decrease (progress) when $\Es_{t}$ occurs}
We need to bound
\begin{align}
    \E[f(x_{t+1}) - f(x_t) | \Es_{t}] \leq -g_\mathrm{thresh}.
\end{align}
By definition, 
\begin{align}
    \Es_{t} = \{ \lVert \nabla f(x_t) \rVert \geq \tau \} \cup \{ \lambda_\mathrm{min}(\nabla^2 f(x_t)) \leq -\sqrt{\rho} \tau^{1/2} \text{ and } \lVert \nabla f(x_t) \rVert \leq \tau \}.
\end{align}
In words, we split $\Es_{t}$ into two cases: either the gradient is large, or we are near a saddlepoint but there is an escape direction.

\subsubsection{Large gradient regime}
If the norm of the gradient is large enough, i.e.
\begin{align}
    \lVert \nabla f(x_t) \rVert^2 \geq \tau^2
\end{align}
then by Corollary~\ref{corollary:decrease2}, 
\begin{align}
    \E[f(x_{t+1})] - f(x_t) \leq - \frac{\eta \tau^2 \lambda_- }{4} \leq -g_\mathrm{thresh}
\end{align}
as long as $\eta \leq \frac{4 \lambda_- \tau^2}{9 L \noiseub}$
and $g_\mathrm{thresh} \leq \frac{\eta \tau^2 \lambda_- }{4}$. For our choice of $\eta = O(\tau^{5/2})$ and $g_\mathrm{thresh} = \tilde O(\tau^5)$, each of these will hold for small enough $\tau$.

\subsubsection{Sharp negative curvature regime}
We start at a point $x_0$ around which we base our Hessian approximation:
\begin{align}
    g(x) = f(x_0) + (x - x_0)^T \nabla f(x_0) + \frac12 (x - x_0)^T H (x - x_0)
\end{align}
where we write $H = \nabla^2 f(x_0)$. We will also write $A = \E[g_0 g_0^T]^{1/2}$ as the preconditioner at $x_0$.
\begin{lemma}[\citet{nesterov2013introductory}]
    For every twice differentiable $\rho$-Hessian Lipschitz function $f$ we have
    \begin{align}
        \lVert \nabla f(x) - \nabla g(x) \rVert \leq \frac\rho2 \lVert x - x_0 \rVert^2.
    \end{align}
\end{lemma}

\begin{lemma}
Suppose that $\lVert \nabla f(x_0) \rVert \leq \tau$. Also suppose the Hessian at $x_0$ has a strong escape direction, i.e. $\lambda_\mathrm{min}(\nabla^2 f(x_0) ) \leq -\sqrt{\rho} \tau^{1/2}$, and define $\gamma = \lambda_- \sqrt{\rho} \tau^{1/2}$ so that $\sqrt{\rho}\tau^{1/2} = \lambda_-^{-1} \gamma$. Then there exists $k < t_\mathrm{thresh}$ so that
\begin{align}
    \E[f(x_{k})] - f(x_0) \leq -f_\mathrm{thresh}
\end{align}
\end{lemma}
\begin{proof}
Suppose not, i.e. suppose that for all $t < t_\mathrm{thresh}$ it holds that
\begin{align}
\label{eq:contradiction-assumption-f-progress}
    \E[f(x_t)] - f(x_0) \geq -f_\mathrm{thresh}.
\end{align}
Under this assumption we will prove bounds which will imply that the assumption cannot hold. In particular, we will give a lower bound on $\E[ \lVert x_t - x_0 \rVert^2]$ that conflicts with Lemma~\ref{lem:iterate-upper-bound}.

Define the following terms:
\begin{align}
    u_{t} &= (I - \eta A H)^t (x_1 - x_0) \\
    \delta_t &= \sum_{i=1}^t (I - \eta A H)^{t-i} A (- \nabla f(x_i) + \nabla g(x_i) ) \\
    d_t &= -\sum_{i=1}^t (I - \eta A H)^{t-i} A \nabla f(x_0) \\
    \zeta_t &= \sum_{i=1}^t (I - \eta A H)^{t-i} \xi_i \\
    \chi_t &= \sum_{i=1}^t (I - \eta A H)^{t-i} (A - A_i) \nabla f(x_i) \\
    \iota_t &= \sum_{i=1}^t (I - \eta A H)^{t-i} (A_i - \hat A_i) \nabla f(x_i).
\end{align}
With the above definitions in hand, we will form a stale Taylor expansion of $f$, and express it in terms of the above terms:
\begin{align*}
    x_{t+1} - x_0 &= x_{t} - x_0 - \eta \hat A_{t} \nabla f(x_{t}) + \eta \xi_{t} \\
    &= x_{t} - x_0 - \eta A_{t} \nabla f(x_{t}) + \eta \xi_{t} + \eta (A_t - \hat A_t) \nabla f(x_t) \\
    &= x_{t} - x_0 - \eta A_{t} \nabla f(x_{t}) + \eta \xi_{t} + \eta \iota_t \\
    &= x_{t} - x_0 - \eta A \nabla f(x_{t}) + \eta \xi_{t} + \eta (A_t - \hat A_t) \nabla f(x_t) + \eta (A - A_t) \nabla f(x_t)  + \eta \iota_t \\
    &= x_{t} - x_0 - \eta A \nabla g(x_{t}) + \eta (-A \nabla f(x_{t}) + A \nabla g(x_{t}) + \xi_{t} ) + \eta (A - A_{t}) \nabla f(x_{t})  + \eta \iota_t \\
    &= x_{t} - x_0 - \eta A \nabla g(x_{t}) + \eta (-A \nabla f(x_{t}) + A \nabla g(x_{t}) + \xi_{t} ) + \eta \chi_t  + \eta \iota_t \\
    &= x_{t} - x_0 - \eta A ( \nabla f(x_0) + H (x_{t} - x_0) ) + \eta (-A \nabla f(x_{t}) + A \nabla g(x_{t}) + \xi_{t} ) + \eta \chi_t  + \eta \iota_t \\
    &= (I - \eta A H) (x_{t} - x_0) + \eta (- A \nabla f(x_{t}) + A \nabla g(x_{t}) - A \nabla f(x_0) + \xi_{t} ) + \eta \chi_t  + \eta \iota_t \\
    &= u_{t} + \eta ( \delta_{t} + d_{t} + \zeta_{t} + \chi_{t} + \iota_t).
\end{align*}

To proceed, we must bound all these terms.

\begin{lemma}
Under the above conditions, we have
\begin{align}
    \E[\lVert \chi_t \rVert] \leq \alpha \tau \sqrt{ \eta^3 L \noiseub \eigone } \cdot \kappa^t \cdot \left( \frac{4}{(\eta \gamma)^2} + \frac{6 f_\mathrm{thresh}}{\eta^3 \gamma L \noiseub} + \frac{2}{\eta \gamma} \cdot  \sqrt{\frac{2r^2}{\eta^3 L \eigone }} \right).
\end{align}
\end{lemma}
\begin{proof}
We assume $A(x)$ is $\alpha$ Lipschitz, so that $\lVert A_i - A \rVert \leq \alpha \lVert x_i - x_0 \rVert$. Then,
\begin{align}
    \E[\lVert \chi_t \rVert]
    &= \E\left[ \left\lVert \sum_{i=1}^t (I - \eta A H)^{t-i} (A - A_i) \nabla f(x_i) \right\rVert \right] \\
    &\leq \sum_{i=1}^t (1 + \eta \gamma)^{t-i} \E\left[ \lVert (A - A_i) \nabla f(x_i) \rVert \right] \\
    &\leq \sum_{i=1}^t (1 + \eta \gamma)^{t-i} \E\left[ \lVert A - A_i \rVert \lVert \nabla f(x_i) \rVert \right] \\
    &\leq \tau \sum_{i=1}^t (1 + \eta \gamma)^{t-i} \E\left[ \lVert A - A_i \rVert \right] \\
    &\leq \alpha \tau \sum_{i=1}^t (1 + \eta \gamma)^{t-i} \E\left[ \lVert x_i - x_0 \rVert \right] \\
    &\leq \alpha \tau \sum_{i=1}^t (1 + \eta \gamma)^{t-i} \sqrt{\E\left[ \lVert x_i - x_0 \rVert^2 \right]} \\
    &\leq \alpha \tau \sum_{i=1}^t (1 + \eta \gamma)^{t-i} \sqrt{ 6 \eta f_\mathrm{thresh} \eigone i + \eta^3 L \noiseub \eigone i^2 + 2 \noiseub r^2 }
\end{align}
where for the last identity we have applied Lemma~\ref{lem:iterate-upper-bound}.
By Lemma~\ref{lem:quadratic-bounds}, we may further bound this by
\begin{align}
    \E[\lVert \chi_t \rVert]
    &\leq \alpha \tau \sqrt{ \eta^3 L \noiseub \eigone } \sum_{i=1}^t (1 + \eta \gamma)^{t-i} \left( 2i + \frac{3 \eta f_\mathrm{thresh} \eigone }{\eta^3 L \noiseub \eigone } + \sqrt{\frac{2 \noiseub r^2}{\eta^3 L \noiseub \eigone }} \right) \\
    &= \alpha \tau \sqrt{ \eta^3 L \noiseub \eigone } \sum_{i=1}^t (1 + \eta \gamma)^{t-i} \left( 2i + \frac{3 f_\mathrm{thresh}}{\eta^2 L \noiseub} + \sqrt{\frac{2r^2}{\eta^3 L \eigone }} \right).
\end{align}
Applying Lemma~\ref{lem:series} with $\beta = \eta \gamma$ yields:
\begin{align}
    \E[\lVert \chi \rVert]
     &\leq \alpha \tau \sqrt{ \eta^3 L \noiseub \eigone } \cdot \kappa^t \cdot \left( \frac{4}{(\eta \gamma)^2} + \frac{2}{\eta \gamma} \cdot \frac{3 f_\mathrm{thresh}}{\eta^2 L \noiseub} + \frac{2}{\eta \gamma} \cdot  \sqrt{\frac{2r^2}{\eta^3 L \eigone }} \right) \\
     &= \alpha \tau \sqrt{ \eta^3 L \noiseub \eigone } \cdot \kappa^t \cdot \left( \frac{4}{(\eta \gamma)^2} + \frac{6 f_\mathrm{thresh}}{\eta^3 \gamma L \noiseub} + \frac{2}{\eta \gamma} \cdot  \sqrt{\frac{2r^2}{\eta^3 L \eigone }} \right).
\end{align}
\end{proof}

\begin{lemma}
Under the above conditions, we have
\begin{align}
    \E[\lVert \delta_t \rVert] \leq \eigtwo \rho \kappa^t \left[\frac{2 \noiseub r^2}{\eta \gamma} + \frac{6 \eta f_\mathrm{thresh} \eigone }{(\eta \gamma)^2} +  \frac{3\eta^3 L \noiseub \eigone }{(\eta \gamma)^3} \right].
\end{align}
\end{lemma}
\begin{proof}
We write
\begin{align}
    \E[\lVert \delta_t \rVert]
    &= \E\left[ \left\lVert \sum_{i=1}^t (I - \eta A H)^{t-i} A (\nabla f(x_i) - \nabla g(x_i) ) \right\rVert \right] \\
    &\leq \sum_{i=1}^t (1 + \eta \gamma)^{t-i} \E\left[ \lVert A (\nabla f(x_i) - \nabla g(x_i) ) \rVert \right] \\
    &\leq \eigtwo  \sum_{i=1}^t \kappa^{t-i} \E\left[ \lVert \nabla f(x_i) - \nabla g(x_i)  \rVert \right] \\
    &\leq \eigtwo (\rho / 2) \sum_{i=1}^t \kappa^{t-i} \E\left[ \lVert x_i - x_0 \rVert^2 \right] \\
    &\leq \eigtwo (\rho / 2) \sum_{i=1}^t \kappa^{t-i} \left(6 \eta f_\mathrm{thresh} \eigone i + \eta^3 L \noiseub \eigone i^2 + 2 \noiseub r^2 \right),
\end{align}
where again, the last inequality comes from Lemma~\ref{lem:iterate-upper-bound}.
Applying Lemma~\ref{lem:series} with $\beta = \eta \gamma$ yields:
\begin{align}
    \E[\lVert \delta_t \rVert]
     &\leq \frac{\eigtwo \rho \kappa^t}{2} \left[ \left(6 \eta f_\mathrm{thresh} \eigone \right) \cdot \frac{2}{\eta^2 \gamma^2}  + \eta^3 L \noiseub \eigone \cdot \frac{6}{\eta^3\gamma^3} + 2 \noiseub r^2 \cdot \frac{2}{\eta \gamma} \right] \\
     &= \eigtwo \rho \kappa^t \left[\frac{2 \noiseub r^2}{\eta \gamma} + \frac{6 \eta f_\mathrm{thresh} \eigone}{(\eta \gamma)^2} +  \frac{3\eta^3 L \noiseub \eigone }{(\eta \gamma)^3} \right].
\end{align}
\end{proof}

\begin{lemma}
Under the above conditions,
\begin{align}
    \E \lVert \iota_t \rVert ] 
    &\leq 2\tau (\eta \gamma)^{-1} \esttol \kappa^t.
\end{align}
\end{lemma}
\begin{proof}
Write
\begin{align}
    \E \lVert \iota_t \rVert ] 
    &= \E \left[ \left\lVert \sum_{i=1}^t (I - \eta A H)^{t-i} (A_i - \hat A_i) \nabla f(x_i) \right\rVert \right] \\
    &\leq \sum_{i=1}^t (1 + \eta \gamma)^{t-i} \E \left[ \left\lVert  (A_i - \hat A_i) \nabla f(x_i) \right\rVert \right] \\
    &\leq \tau \sum_{i=1}^t (1 + \eta \gamma)^{t-i} \E \left[ \left\lVert  A_i - \hat A_i \right\rVert \right] \\
    &\leq 2\tau (\eta \gamma)^{-1} \kappa^t \max_{i} \E \left[ \left\lVert  A_i - \hat A_i \right\rVert \right] \\
    &\leq 2\tau (\eta \gamma)^{-1} \esttol \kappa^t.
\end{align}
\end{proof}

\begin{lemma}
Under the above conditions, $\E[u_t^T] d_t \geq 0.$
\end{lemma}
\begin{proof}
We have
\begin{align}
    \E[u_t] = (I - \eta A H)^t \E[x_1 - x_0] = -r (I - \eta A H)^t A \nabla f(x_0).
\end{align}
For small enough $\eta$, \matt{todo: be more precise about how small $\eta$ must be}
we have $\lVert \eta A H \rVert \leq 1$ and hence:
\begin{align}
    \E[u_t^T] d_t &= r \left[ (I - \eta A H)^t A \nabla f(x_0) \right]^T \sum_{i=1}^t (I - \eta A H)^{t-i} A \nabla f(x_0) \\
    &= r \sum_{i=1}^t (A \nabla f(x_0) )^T (I - \eta A H)^{2t-i} (A \nabla f(x_0)) \geq 0.
\end{align}
\end{proof}

\begin{lemma}
Under the above conditions, we get an exponentially growing lower bound on the expected squared norm of $u_t$:
\begin{align}
    \E[\lVert u_t \rVert^2] &\geq (1 + \eta \gamma)^{2t} r^2 \noiselb = \kappa^{2t} r^2 \noiselb.
\end{align}
\end{lemma}
\begin{proof}
For unit vectors $v$, we may write
\begin{align}
    \E[\lVert u_t \rVert^2]
    &\geq \E[( v^T u_t)^2].
\end{align}
In particular, by definition of $u_t$,
\begin{align}
    \E[\lVert u_t \rVert^2]
    &\geq \E[( v^T (I - \eta A H)^t (x_1 - x_0 ))^2].
\end{align}
We wish to choose a unit vector $v$ so that this is as large as possible. If $AH$ were symmetric, we could choose $v$ to be an eigenvector, but the product of symmetric matrices is not in general symmetric. However, because $A$ and $H$ are both symmetric, and $A$ is positive definite, it follows that $A^{1/2}$ exists and that $A^{1/2} H A^{1/2}$ is symmetric. Hence for orthonormal $U$ and diagonal $\Lambda$, we have 
\begin{align}
    A^{1/2} H A^{1/2} &= U \Lambda U^T \\
    \implies A^{-1/2} A H A^{1/2} &= U \Lambda U^T \\
    \implies AH &= A^{1/2} U \Lambda (A^{1/2} U)^{-1}.
\end{align}
The diagonal matrix $\Lambda$ contains the eigenvalues of $A^{1/2} H A^{1/2}$. Without loss of generality, $\Lambda_{11}$ corresponds to a negative eigenvalue with absolute value $\gamma$.
Therefore
\begin{align}
    (I - \eta AH)^t 
    &= ( A^{1/2} U (I - \eta \Lambda) (A^{1/2} U)^{-1} )^t \\
    &= A^{1/2} U (I - \eta \Lambda)^t (A^{1/2} U)^{-1}.
\end{align}
Since we can choose $v$ to be any unit vector we want, we will set it equal to $C (U^T A^{1/2})^{-1} e_1$ so that $U^T A^{1/2} v = C e_1$. Here $e_1$ is the first standard basis vector and $C$ is a scalar constant chosen to make $v$ a unit vector. Taking transposes, we have $v^T A^{1/2} U = C e_1^T$. Now,
\begin{align}
    v^T (I - \eta A H)^t
    &= v^T A^{1/2} U (I - \eta \Lambda)^t (A^{1/2} U)^{-1} \\
    &= C e_1^T (I - \eta \Lambda)^t (A^{1/2} U)^{-1} \\
    &= C (1 + \eta \Lambda_{11})^t e_1^T (A^{1/2} U)^{-1} \\
    &= (1 + \eta \gamma)^t \cdot C e_1^T (A^{1/2} U)^{-1}.
\end{align}
Substituting in the definition of $v$, this is equal to:
\begin{align}
    v^T (I - \eta A H)^t
    &= (1 + \eta \gamma)^t \cdot v^T (A^{1/2} U) (A^{1/2} U)^{-1} \\
    &= (1 + \eta \gamma)^t v^T.
\end{align}
This equality holds for any $v$ of the form specified above; in particular, choose $C$ so that $v$ is unit. Then, we may finally bound
\begin{align}
    \E[\lVert u_t \rVert^2]
    &\geq \E[( v^T (I - \eta A H)^t (x_1 - x_0 ))^2] \\
    &\geq (1 + \eta \gamma)^{2t} \E[( v^T (x_1 - x_0 ))^2] \\
    &= (1 + \eta \gamma)^{2t} r^2 \E[( v^T A g_0)^2] \\
    &= (1 + \eta \gamma)^{2t} r^2 v^T \E[A g_0 g_0^T A^T] v \\
    &= (1 + \eta \gamma)^{2t} r^2 v^T A \E[g_0 g_0^T] A^T v \\
    &\geq (1 + \eta \gamma)^{2t} r^2 \lambda_\mathrm{min}(A \E[g_0 g_0^T] A^T) \\
    &\geq (1 + \eta \gamma)^{2t} r^2 \noiselb,
\end{align}

where the last two lines follow by the fact that $\lVert v \rVert=1$ and by definition of $\noiselb$.
\end{proof}

\begin{lemma}
Under the above conditions we have a deterministic bound on $\lVert u_t \rVert$:
\begin{align}
    \lVert u_t \rVert \leq \kappa^t r M 
\end{align}
\end{lemma}
\begin{proof}
We write
\begin{align}
    \lVert u_t \rVert 
    &= \lVert (I - \eta A H)^t (x_1 - x_0) \rVert \\
    &\leq \lVert I - \eta A H \rVert^t \cdot \lVert x_1 - x_0 \rVert \\
    &\leq (1 + \eta \gamma)^t \cdot r \lVert A g_0 \rVert \\
    &\leq (1 + \eta \gamma)^t \cdot r M. 
\end{align}
\end{proof}

Putting all these results together, we can give a lower bound on the distance between iterates:
\begin{align*}
    \E[ \lVert x_{t+1} - x_0 \rVert^2]
    &= \E\left[ \lVert u_t + \eta (\delta_t + d_t + \zeta_t + \chi_t + \iota_t) \rVert^2 \right] \\
    &= \E[ \lVert u_t \rVert^2 ] + 2 \eta \E\left[ u_t^T (\delta_t + d_t + \zeta_t + \chi_t + \iota_t) \right] + \eta^2 \E\left[ \lVert \delta_t + d_t + \zeta_t + \chi_t + \iota_t \rVert^2 \right] \\
    &\geq \E[ \lVert u_t \rVert^2 ] + 2 \eta \E\left[ u_t^T (\delta_t + d_t + \zeta_t + \chi_t + \iota_t) \right] \\
    &= \E[ \lVert u_t \rVert^2 ] + 2 \eta \E\left[ u_t^T (\delta_t + d_t + \chi_t + \iota_t) \right] \\
    &= \E[ \lVert u_t \rVert^2 ] + 2 \eta \E[ u_t^T \delta_t] + 2\eta\E[u_t^T d_t] + 2\eta\E[u_t^T \chi_t ] \\
    &= \E[ \lVert u_t \rVert^2 ] + 2 \eta \E[ u_t^T \delta_t] + 2\eta\E[u_t]^T d_t + 2\eta\E[u_t^T \chi_t ] + 2\eta\E[u_t^T \iota_t] \\
    &\geq \E[ \lVert u_t \rVert^2 ] + 2 \eta \E[ u_t^T \delta_t] + 2\eta\E[u_t^T \chi_t ] + 2\eta\E[u_t^T \iota_t ]\\
    &\geq \E[ \lVert u_t \rVert^2 ] - 2 \eta \lVert u_t \rVert \E[ \lVert \delta_t \rVert ] - 2\eta \lVert u_t \rVert \E[ \lVert \chi_t \rVert ] - 2 \eta \lVert u_t \rVert \E[\lVert \iota_t \rVert] \\
    &\geq \kappa^{2t} r^2 \noiselb - 2\eta \kappa^t r M \E[ \lVert \delta_t \rVert + \lVert \chi_t \rVert + \lVert \iota_t \rVert].
\end{align*}

Substituting in the bounds for $\E[ \lVert \delta_t \rVert]$, $\E[ \lVert \chi_t \rVert]$, and $\E[ \lVert \iota_t \rVert ]$,we finally have the lower bound:
\begin{align}
    \left( r \noiselb - 2\eta M \left[ \eigtwo \rho \left[\frac{2 \noiseub r^2}{\eta \gamma} + \frac{6 \eta f_\mathrm{thresh} \eigone} {(\eta \gamma)^2} +  \frac{3\eta^3 L \noiseub \eigone }{(\eta \gamma)^3} \right] \right. \right. \\
    \left. \left. + \alpha \tau \sqrt{ \eta^3 L \noiseub \eigone }  \left( \frac{4}{(\eta \gamma)^2} + \frac{6 f_\mathrm{thresh}}{\eta^3 \gamma L \noiseub} + \frac{2}{\eta \gamma} \cdot  \sqrt{\frac{2r^2}{\eta^3 L \eigone }} \right) + 2 \tau (\eta \gamma)^{-1} \esttol  \right] 
    \right) r \kappa^{2t}.
\end{align}
As long as the sum in the parentheses is positive, this term will grow exponentially and grant us the contradiction we seek. We want to bound each of the seven terms in brackets by $r \noiselb / 8$, so that the overall bound is $r^2 \kappa^{2t} \noiselb / 8$. 
For simplicity, we will write $K=1/8$ as a universal constant.
Then, we want to choose parameters so the following inequalities all hold. 

We start with the last term (from $\iota_t$) because it is the most simple. Since $\gamma = \Theta(\tau^{1/2})$, we require that
\begin{align}
    2 \eta M \cdot 2 \tau (\eta \gamma)^{-1} \esttol &\leq r \noiselb K \\
    \Leftrightarrow 4 M \tau  \gamma^{-1} \esttol &\leq r \noiselb K \\
    \Leftrightarrow \tau \cdot \tau^{-1/2} \esttol &\leq O(r) \\
    \Leftrightarrow  \esttol &\leq O(\tau^{-1/2} r ).
\end{align}
Since we will eventually set $r = O(\tau)$, this constraint is simply $\esttol \leq O( \tau^{1/2} )$.

Next we move onto the first three terms, which correspond to $\delta_t$:
\begin{align}
    2 \eta M \eigtwo \rho \cdot \frac{2 \noiseub r^2}{\eta \gamma} \leq r \noiselb K &\Leftrightarrow r \leq \frac{\gamma \noiselb K}{4 \eigtwo \noiseub \rho M} \label{constraint:r-upper-bound} \\
    2 \eta M \eigtwo \rho \cdot \frac{6 \eta f_\mathrm{thresh} \eigone}{\eta^2 \gamma^2} \leq r \noiselb K &\Leftrightarrow f_\mathrm{thresh} \leq \frac{\gamma^2 r \noiselb K}{12 \eigone \eigtwo \rho M} \label{constraint:fthresh-upper-bound} \\
    2 \eta M \eigtwo \rho \cdot \frac{3\eta^3 L \noiseub \eigone }{\eta^3 \gamma^3} \leq r \noiselb K &\Leftrightarrow \eta \leq \frac{\gamma^3 r \noiselb K}{6 M L \eigone \eigtwo \noiseub \rho}. \label{constraint:eta-upper-bound}
\end{align}
The first constraint is satisfied for small enough $\tau$ because we chose $r = O(\tau) \leq O(\tau^{1/2})$.
The second term is equivalent to
\begin{align}
    f_\mathrm{thresh} &\stackrel{?}{\leq} \frac{\gamma^2 \noiselb K}{12 \eigone \eigtwo \rho M} \cdot r\\
    \Leftrightarrow \gamma^4 \cdot \frac{\delta \noiselb^2 K^2}{54 \cdot 12 \eigone^2 \eigtwo^2 \noiseub L \rho^2 M^2 } &\stackrel{?}{\leq} \frac{\gamma^2 \noiselb K}{12 \eigone \eigtwo \rho M} \cdot \gamma^2 \cdot \frac{\delta \noiselb K}{54 \eigone \eigtwo \noiseub L \rho M}\\
    \Leftrightarrow \frac{\delta \noiselb^2 K^2}{54 \cdot 12 \eigone^2 \eigtwo^2 \noiseub L \rho^2 M^2 }  &\stackrel{?}{\leq} \frac{\delta \noiselb^2 K^2}{54 \cdot 12 \eigone^2 \eigtwo^2 \noiseub L \rho^2 M^2}
\end{align}
which trivially always holds since the two expressions are equal.


Finally, we address the three terms corresponding to $\chi_t$. For small enough $\tau$, it will turn out that none of the resulting constraints are tight, i.e. they are all weaker than some other constraint we already require.
First,
\begin{align}
        2 \eta M \alpha \tau \sqrt{\eta^3 L \noiseub \eigone } \cdot \frac{4}{\eta^2 \gamma^2} &\leq r \noiselb K \\
        \Leftarrow \eta^{1/2} \tau  &\leq O(r \gamma^2) \\
        \Leftrightarrow \eta &\leq O(r^2 \gamma^4 \tau^{-2}) = O(\tau^2).
\end{align}
Next,
\begin{align}
    2 \eta M \alpha \tau \sqrt{\eta^3 L \noiseub \eigone } \cdot \frac{6 f_\mathrm{thresh} }{\eta^3 \gamma L \noiseub} &\leq r \noiselb K \\
    \Leftarrow \eta \tau \eta^{3/2} \frac{ f_\mathrm{thresh} }{\eta^3 \gamma } &\leq O(r) \\
    \Leftrightarrow f_\mathrm{thresh}  &\leq O(\eta^{1/2} r \gamma \tau^{-1}) = O(\tau^{7/4}).
\end{align}
Finally,
\begin{align}
    2 \eta M \alpha \tau \sqrt{\eta^3 L \noiseub \eigone } \cdot \frac{2}{\eta \gamma} \cdot  \sqrt{\frac{2r^2}{\eta^3 L \eigone }} &\leq r \noiselb K \\
    \Leftarrow \tau \sqrt{\eta^3 } \cdot \frac{1}{ \gamma} \cdot  \frac{r}{\sqrt{\eta^3  }} &\leq O(r) \\
    \Leftrightarrow \tau \gamma^{-1} r &\leq O(r) \\
    \Leftrightarrow \tau &\leq O(\gamma) = O(\tau^{1/2}).
\end{align}

Hence, for small enough $\tau$, for the above parameter settings, we have
\begin{align}
    \E[ \lVert x_{t+1} - x_0 \rVert^2 ] \geq r^2 \kappa^{2t} \noiselb K.
\end{align}

We now have a lower bound and an upper bound that when combined yield $(1 + \eta \gamma)^{2t} \leq C$, where
\begin{align}
    C &= \left[ \left(6 \eta f_\mathrm{thresh} \eigone \right)t + \eta^3 L \noiseub \eigone t^2 + 2 \noiseub r^2 \right] \cdot \frac{1}{r^2 \noiselb K}.
\end{align}
We can choose $\omega$ that is only logarithmic in all parameters, i.e. $\omega = O(\log(\frac{\eigone \eigtwo \noiseub L \eta f_\mathrm{thresh}}{\noiselb r}))$, so that
setting $t \geq t_\mathrm{thresh} = \omega / (\eta \gamma)$ yields $(1+\eta \gamma)^{2t} \geq C$. This contradicts the upper bound, as desired.

\end{proof}

\begin{lemma}
\label{lem:iterate-upper-bound}
Assume that Equation~\eqref{eq:contradiction-assumption-f-progress} holds. Assume also that
$\eta \leq \frac{f_\mathrm{thresh} \eigone}{\noiseub}$.
Then,
\begin{align}
    \E[ \lVert x_t - x_0 \rVert^2 ]
    &\leq 6 \eta f_\mathrm{thresh} \eigone t + \eta^3 L \noiseub \eigone t^2 + 2 \noiseub r^2.
\end{align}
\end{lemma}
\begin{proof}
By Lemma~\ref{lem:descent},
\begin{align}
    -f_\mathrm{thresh} &\leq \E[f(x_{t})] - f(x_0) \\
    &= \E\left[ \sum_{i=0}^{t-1} f(x_{i+1}) - f(x_i) \right] \\
    &\leq -\eta \sum_{i=0}^{t-1} \E[ \lVert \hat A_i^{1/2} \nabla f(x_i) \rVert^2 ] + \frac{\eta^2 L \noiseub (t-1)}{2} + \frac{r^2 L \noiseub}{2}.
\end{align}
Remember, we are making the simplifying assumption that $\eigone$ serves as a bound in the same way for $\hat A$ as it does for $A$. 
This is trivially true if $\esttol = 0$.
Applying the definition of $\eigone$ yields:
\begin{align}
    -f_\mathrm{thresh} &\leq -\eta \eigone^{-1} \sum_{i=0}^{t-1} \E[ \lVert \hat A_i \nabla f(x_i) \rVert^2 ] + \frac{\eta^2 L \noiseub t}{2} + \frac{r^2 L \noiseub}{2}.
\end{align}
By rearranging, we can get a bound on the gradient norms:
\begin{align}
    \sum_{i=0}^{t-1} \E[ \lVert \hat A_i \nabla f(x_i) \rVert^2 ] 
    &\leq \frac{\eigone}{\eta} \left( \frac{\eta^2 L \noiseub t}{2} + \frac{r^2 L \noiseub}{2} + f_\mathrm{thresh} \right) \\
    &= \frac{\eta L \noiseub \eigone t}{2} + \frac{r^2 L \noiseub \eigone}{2\eta} + \frac{f_\mathrm{thresh} \eigone}{\eta}.
    \label{eq:grad-norm-bound-init}
\end{align}
Before we proceed, note that 
\matt{todo: figure out what eq I was referring to}
we already have
\begin{align}
    \frac{\delta f_\mathrm{thresh}}{4} \geq \frac{9 L \noiseub r^2}{8} \implies \frac{f_\mathrm{thresh} \eigone}{\eta} \geq \frac{9}{2\delta} \frac{ r^2 L \noiseub \eigone}{\eta} \geq \frac{ r^2 L \noiseub \eigone}{2\eta}.
\end{align}
Hence we can further bound equation~\eqref{eq:grad-norm-bound-init} by
\begin{align}
    \sum_{i=0}^{t-1} \E[ \lVert \hat A_i \nabla f(x_i) \rVert^2 ] 
    &\leq \frac{\eta L \noiseub \eigone t}{2} + \frac{2 f_\mathrm{thresh} \eigone}{\eta}.
    \label{eq:grad-norm-bound}
\end{align}

Now we will work toward bounding the norm of the difference $x_t - x_0$. We will first bound the difference $x_t - x_1$, then the difference $x_1 - x_0$.
\begin{align}
    \E[ \lVert x_t - x_1 \rVert^2 ] 
    &\leq \E \left[ \left\lVert \sum_{i=1}^{t-1} x_{i+1} - x_i \right\rVert^2 \right] \\
    &\leq \eta^2 \E \left[ \left\lVert \sum_{i=1}^{t-1} (\xi_i - \hat A_i \nabla f(x_i)) \right\rVert^2 \right],
    \label{eq:bound-iterate-dist-by-grad}
\end{align}
where $\xi_i = \hat A_i (\nabla f(x_i) - g_i)$ is the zero mean effective noise that arises from rescaling the stochastic gradient noise. We may write
\begin{align}
    \E \left[ \left\lVert \sum_{i=1}^{t-1} (\xi_i - \hat A_i \nabla f(x_i)) \right\rVert^2 \right]
    &= \E \left[ \left\lVert \sum_{i=1}^{t-1} \xi_i - \sum_{i=1}^{t-1} \hat A_i \nabla f(x_i) \right\rVert^2 \right] \\
    &= \E \left[ \left\lVert \sum_{i=1}^{t-1} \hat A_i \nabla f(x_i) \right\rVert^2 + \left\lVert \sum_{i=1}^{t-1} \xi_i \right\rVert^2  - 2 \sum_{i=1}^{t-1} \sum_{j=1}^{t-1} \langle \xi_i, \hat A_j \nabla f(x_j) \rangle \right] \\
    &= \E \left[ \left\lVert \sum_{i=1}^{t-1} \hat A_i \nabla f(x_i) \right\rVert^2 \right] + \E \left[ \left\lVert \sum_{i=1}^{t-1} \xi_i \right\rVert^2   \right]
\end{align}
because $\xi_i$ are zero mean. 
Since $\E[\xi_i^T \xi_j] = 0$ for $i \not = j$, the expression can be simplified as:
\begin{align}
    \E \left[ \left\lVert \sum_{i=1}^{t-1} (\xi_i - \hat A_i \nabla f(x_i)) \right\rVert^2 \right]
    &= \E \left[ \left\lVert \sum_{i=1}^{t-1} \hat A_i \nabla f(x_i) \right\rVert^2 \right] + \sum_{i=1}^{t-1} \E \left[ \left\lVert \xi_i \right\rVert^2 \right] \\
    &\leq \E \left[ \left\lVert \sum_{i=1}^{t-1} \hat A_i \nabla f(x_i) \right\rVert^2 \right] + \sum_{i=1}^{t-1} \E \left[ \left\lVert \xi_i \right\rVert^2 \right] \\
    &\leq \E \left[ \left(  \sum_{i=1}^{t-1} \left\lVert \hat A_i \nabla f(x_i) \right\rVert \right)^2 \right] + \sum_{i=1}^{t-1} \E \left[ \left\lVert \xi_i \right\rVert^2 \right] \\
    &\leq (t-1) \sum_{i=1}^{t-1} \E \left[\left\lVert \hat A_i \nabla f(x_i) \right\rVert^2 \right] + \sum_{i=1}^{t-1} \E \left[ \left\lVert \xi_i \right\rVert^2 \right].
\end{align}
Note 
\begin{align}
    \E[ \lVert \xi_i \rVert^2] 
    &\leq \E[ \lVert \hat A_i \nabla f(x_i) \rVert^2 ] + \E[ \lVert \hat A_i g_i \rVert^2] \\
    &\leq \E[ \lVert \hat A_i \nabla f(x_i) \rVert^2 ] + \frac94 \noiseub 
\end{align}
where we have used Lemma~\ref{lem:c3-error-bound}.
We can then bound
\begin{align}
    \E \left[ \left\lVert \sum_{i=1}^{t-1} (\xi_i - \hat A_i \nabla f(x_i)) \right\rVert^2 \right]
    &\leq (t-1+1)  \sum_{i=1}^{t-1} \E \left[\left\lVert \hat A_i \nabla f(x_i) \right\rVert^2 \right] + \frac{9 t \noiseub}{4}.
\end{align}
Plugging in Equation~\eqref{eq:grad-norm-bound} we get:
\begin{align}
    \E \left[ \left\lVert \sum_{i=1}^{t-1} (\xi_i - \hat A_i \nabla f(x_i)) \right\rVert^2 \right]
    &\leq t \left( \frac{\eta L \noiseub \eigone t}{2} + \frac{2 f_\mathrm{thresh} \eigone}{\eta} \right) + t \noiseub.
\end{align}
Plugging this into Equation~\eqref{eq:bound-iterate-dist-by-grad} yields:
\begin{align}
    \E[ \lVert x_t - x_1 \rVert^2 ] 
    &\leq t \eta^2 \left( \frac{\eta L \noiseub \eigone t}{2} + \frac{2 f_\mathrm{thresh} \eigone }{\eta} \right) + \eta^2 \noiseub t \\
    &= \left(4 \eta f_\mathrm{thresh} \eigone + \eta^2 \noiseub \right)t + \frac{\eta^3 L \noiseub \eigone t^2}{2}.
\end{align}
Then we may write
\begin{align}
    \E[ \lVert x_t - x_0 \rVert^2 ]
    &\leq 2 \E[ \lVert x_t - x_1 \rVert^2 ] + 2 \E[ \lVert x_1 - x_0 \rVert^2 ] \\
    &\leq \left(4 \eta f_\mathrm{thresh} \eigone +  2 \eta^2 \noiseub \right)t + \eta^3 L \noiseub \eigone t^2 + 2 \noiseub r^2.
\end{align}

We are almost done. 
By our additional assumption that $\eta \leq \frac{f_\mathrm{thresh} \eigone}{\noiseub}$ (which will wind up being true for small enough $\tau$), it also follows that
\begin{align}
    2 \eta^2 \noiseub \leq 2 \eta f_\mathrm{thresh} \eigone
\end{align}
and therefore
\begin{align}
    \E[ \lVert x_t - x_0 \rVert^2 ]
    &\leq 6 \eta f_\mathrm{thresh} \eigone  t + \eta^3 L \noiseub \eigone t^2 + 2 \noiseub r^2.
\end{align}

\end{proof}

\subsection{Auxiliary lemmas}

\begin{lemma}
\label{lem:quadratic-bounds}
For $z, A, B, C \geq 0$,
\begin{align}
    \sqrt{A z^2 + B z + C} \leq \sqrt{A} \cdot \left(2z + \frac{B}{2A} + \sqrt{\frac{C}{A}} \right).
\end{align}
\end{lemma}
\begin{proof}
Note the following two facts:
\begin{align}
    A z^2 + B z + C = A (z^2 + B/A z + C/A) = A [(z + B/(2A))^2 + C/A - B^2/(2A)^2]
\end{align}
and
\begin{align}
    A z^2 + B z + C = A (z^2 + B/A z + C/A) = A[ (z + \sqrt{C/A})^2 - 2 \sqrt{C/A} + B/A].
\end{align}
If $B^2 \geq 4AC$, then $C/A - B^2/(2A)^2 \leq 0$. Otherwise, $-2\sqrt{C/A} + B/A \leq 0$.
Hence,
\begin{align}
    \sqrt{A z^2 + B z + C} &\leq \begin{cases}
        \sqrt{A} \cdot (z + B/(2A)) & \text{ case 1} \\
        \sqrt{A} \cdot (z + \sqrt{C/A}) & \text{ case 2.}
    \end{cases} \\
    &\leq \sqrt{A} \cdot \left[ (z + B/(2A)) + (z + \sqrt{C/A}) \right].
\end{align}
\end{proof}

\begin{lemma}
\label{lem:exponential-growth-bound}
Let $0 < x < 1$.
For $t \geq 2 \log C / x$, we have $(1+x)^t \geq C$.
\end{lemma}
\begin{proof}
For $x < 1$ we have $\log(1+x) \leq x - x^2/ 2 \leq x / 2$. Hence,
\begin{align}
    t \log (1 + x)
    &\geq t x / 2 \\
    &\geq \log C,
\end{align}
and the lemma follows by exponentiating both sides.
\end{proof}


\subsubsection{Series lemmas}
\begin{lemma}[As in~\citet{pmlr-v80-daneshmand18a}]
\label{lem:series}
For $0 < \beta < 1$ the following inequalities hold:
\begin{align}
    \sum_{i=1}^t (1+\beta)^{t-i} &\leq 2\beta^{-1} (1+\beta)^t \\
    \sum_{i=1}^t (1+\beta)^{t-i} i &\leq 2\beta^{-2} (1+\beta)^t \\
    \sum_{i=1}^t (1+\beta)^{t-i} i^2 &\leq 6\beta^{-3} (1+\beta)^t.
\end{align}
\end{lemma}

\subsection{Descent lemmas}
First we need a quick lemma relating the constants of the true preconditioner to those of an approximate preconditioner:
\begin{lemma}
    Let $\noiseub$ be an upper bound on $\E[\lVert A g \rVert^2]$. Let $\hat A$ be another matrix with $\lVert \hat A - A \rVert \leq \esttol < \lambda_- / 2$. Then, $\E[\lVert \hat A g \rVert^2] \leq \frac94 \noiseub.$
    \label{lem:c3-error-bound}
\end{lemma}
\begin{proof}
The proof is straightforward:
\begin{align}
    \E[\lVert \hat A g \rVert^2]
    &\leq \E[\lVert (A + \esttol I) g \rVert^2] \\
    &\leq \E \left[ \left\lVert \frac32 A g \right\rVert^2 \right] \\
    &= \frac94 \E[ \left\lVert A g \right\rVert^2] 
    = \frac94 \noiseub
\end{align}
where the penultimate line follows by $\esttol < \lambda_- / 2$ and $\esttol I \preceq \frac12 A_t$.
\end{proof}

Note that in the noiseless case $\esttol = 0$, all the below results still apply, and we only lose a constant factor compared to the typical descent lemma.

\begin{lemma}
\label{lem:descent}
Assume $f$ has $L$-Lipschitz gradient. Suppose we perform the updates $x_{t+1} \leftarrow x_t - \eta \hat A_t g_t$, where $g_t$ is a stochastic gradient, $A_t$ is a $(\eigone,\eigtwo,\noiseub,\noiselb,\lambda_-)$-preconditioner, and $\lVert \hat A_t - A_t \rVert \leq \esttol < \frac{\lambda_-}{2}$.
Then,
\begin{align}
    \E[f(x_{t+1})] &\leq f(x_t) - \frac{\eta \lambda_-}{2} \lVert \nabla f(x_t) \rVert^2 + \frac{9 \eta^2 L \noiseub}{8}
\end{align}
\end{lemma}
\begin{proof}
We write
\begin{align}
    \E[f(x_{t+1})] &\leq f(x_t) + \langle \nabla f(x_t), \E[ x_{t+1} - x_t ] \rangle + \frac{L}{2} \E[ \lVert x_{t+1} - x_t \rVert^2 ]\\
    &= f(x_t) - \eta \langle \nabla f(x_t), \hat A_t \nabla f(x_t) \rangle + \frac{\eta^2 L}{2} \E\left[ \lVert \hat A_t g_t \rVert^2 \right] \\
    &\leq f(x_t) - \eta (\lambda_- - \esttol) \lVert \nabla f(x_t) \rVert^2 + \frac{9 \eta^2 L \noiseub}{8}  \\
    &\leq f(x_t) - \frac{\eta \lambda_-}{2} \lVert \nabla f(x_t) \rVert^2 + \frac{9 \eta^2 L \noiseub}{8}
\end{align}
where the third line follows by Lemma~\ref{lem:c3-error-bound}.
\end{proof}

\begin{corollary}
\label{corollary:increase-not-so-bad}
    Always
    \begin{align}
        \E[f(x_1)] - f(x_0) \leq \frac{9 \eta^2 L \noiseub}{8}.
    \end{align}
\end{corollary}
\begin{corollary}
\label{corollary:sufficient-decrease}
    Suppose $\eta \leq 4 \lambda_- \lVert \nabla f(x_0) \rVert^2 / ( 9 L \noiseub)$. Then,
    \begin{align}
        \E[f(x_1)] - f(x_0) \leq -\frac{\eta  \lambda_-}{4} \lVert \nabla f(x_0) \rVert^2.
    \end{align}
\end{corollary}
\begin{corollary}
\label{corollary:decrease2}
    Suppose $\lVert \nabla f(x_0) \rVert^2 \geq \tau^2$. Then if $\eta \leq 4 \lambda_- \tau^2 / (9 L \noiseub)$
    \begin{align}
        \E[f(x_1)] - f(x_0) &\leq -\frac{\eta  \lambda_-}{4} \lVert \nabla f(x_0) \rVert^2
        \leq -\frac{\eta  \lambda_-}{4} \tau^2.
    \end{align}
\end{corollary}


\section{Convergence to First-Order Stationary Points}
\label{appendix:first-order}

\subsection{Generic Preconditioners: Proof of Theorem~\ref{thm:first-order-no-error}}

\begin{proof}
Let $g$ be the stochastic gradient at time $t$. 
We will precondition by $A_t = A(x_t)$.
We write
\begin{align}
    \E[f(x_{t+1})] &\leq f(x_t) + \langle \nabla f(x_t), \E[ x_{t+1} - x_t ] \rangle + \frac{L}{2} \E[ \lVert x_{t+1} - x_t \rVert^2 ]\\
    &= f(x_t) - \eta \langle \nabla f(x_t), A_t \nabla f(x_t) \rangle + \frac{\eta^2 L}{2} \E\left[ \lVert A_t g_t \rVert^2 \right] \\
    &\leq f(x_t) - \eta \langle \nabla f(x_t), A_t \nabla f(x_t) \rangle + \frac{\eta^2 L \noiseub}{2} \\
    &\leq f(x_t) - \eta \lambda_\mathrm{min}(A_t) \lVert \nabla f(x_t) \rVert^2 + \frac{\eta^2 L \noiseub}{2} \\
    &\leq f(x_t) - \eta \lambda_- \lVert \nabla f(x_t) \rVert^2 + \frac{\eta^2 L \noiseub}{2}.
\end{align}
Summing and telescoping, we have
\begin{align}
    \E[f(x_T)] &\leq \E[f(x_0)] - \eta \lambda_- \sum_{t=0}^{T-1} \E\left[ \lVert \nabla f(x_t) \rVert^2 \right] + \frac{\eta^2 L T \noiseub}{2}.
\end{align}
Now rearrange, and bound $f(x_T)$ by $f^*$ to get:
\begin{align}
    \frac1T \cdot \lambda_- \cdot \sum_{t=0}^{T-1} \E\left[ \lVert \nabla f(x_t) \rVert^2 \right] &\leq \frac{f(x_0) - f^*}{T \eta}
    + \frac{\eta L \noiseub}{2}.
\end{align}
and therefore
\begin{align}
    \frac1T \sum_{t=0}^{T-1} \E\left[ \lVert \nabla f(x_t) \rVert^2 \right] &\leq \left( \frac{f(x_0) - f^*}{T \eta}
    + \frac{\eta L \noiseub}{2} \right) \cdot \frac{1}{\lambda_-}.
\end{align}
Optimally choosing $\eta = \sqrt{2 (f(x_0) - f^*) / (T L \noiseub)}$ yields the overall bound 
\begin{align}
    \frac{1}{T} \sum_{t=0}^{T-1} \E\left[ \lVert \nabla f(x_t) \rVert^2 \right] 
    &\leq \sqrt{\frac{2(f(x_0) - f^*) L \noiseub  }{T}} \cdot \frac{1}{\lambda_-}.
\end{align}
Rephrasing, in order to be guaranteed that the left hand term is bounded by $\tau^2$, it suffices to choose $T$ so that
\begin{align}
    \sqrt{\frac{2(f(x_0) - f^*) L \noiseub  }{T}} \cdot \frac{1}{\lambda_-} \leq \tau^2 \\
    \Leftrightarrow T \geq  \frac{2(f(x_0) - f^*) L \noiseub  }{ \tau^4 \lambda_-^2} 
\end{align}
and
\begin{align}
    \eta &= \sqrt{\frac{2 (f(x_0) - f^*) }{ T L \noiseub}} \\
    &\leq \sqrt{\frac{2 (f(x_0) - f^*) }{ L \noiseub} \cdot \frac{ \tau^4 \lambda_-^2}{2(f(x_0) - f^*) L \noiseub  }} 
    = \frac{ \tau^2 \lambda_-}{L \noiseub  }.
\end{align}
\end{proof}

\subsection{Generic Preconditioners with Errors: Proof of Theorem~\ref{thm:first-order-with-error}}
\begin{proof}
Let $g$ be the stochastic gradient at time $t$. 
We will precondition by $\hat A_t$ which satisfies $\lVert \hat A_t - A_t \rVert \leq \esttol < \lambda_- / 2$.
We write
\begin{align}
    \E[f(x_{t+1})] &\leq f(x_t) + \langle \nabla f(x_t), \E[ x_{t+1} - x_t ] \rangle + \frac{L}{2} \E[ \lVert x_{t+1} - x_t \rVert^2 ]\\
    &= f(x_t) - \eta \langle \nabla f(x_t), \hat A_t \nabla f(x_t) \rangle + \frac{\eta^2 L}{2} \E\left[ \lVert \hat A_t g_t \rVert^2 \right] \\
    &\leq f(x_t) - \eta (\lambda_- - \esttol) \lVert \nabla f(x_t) \rVert^2 + \frac{\eta^2 L}{2} \E\left[ \lVert (A_t + \esttol I) g_t \rVert^2 \right] \\
    &\leq f(x_t) - \frac{\eta \lambda_-}{2} \lVert \nabla f(x_t) \rVert^2 + \frac{\eta^2 L}{2} \E\left[ \left\lVert \frac32 A_t g_t \right\rVert^2 \right] \\
    &= f(x_t) - \frac{\eta \lambda_-}{2} \lVert \nabla f(x_t) \rVert^2 + \frac{9 \eta^2 L \noiseub}{8}
\end{align}
where the penultimate line follows by $\esttol < \lambda_- / 2$ and $\esttol I \preceq \frac12 A_t$.
Summing and telescoping, and further bounding $9/8 < 2$, we have
\begin{align}
    \E[f(x_T)] &\leq \E[f(x_0)] - \frac{\eta \lambda_-}{2} \sum_{t=0}^{T-1} \E\left[  \lVert \nabla f(x_t) \rVert^2 \right] + 2\eta^2 L \noiseub.
\end{align}
Now rearrange, and bound $f(x_T)$ by $f^*$ to get:
\begin{align}
    \frac1T \cdot \frac{\lambda_-}{2} \cdot \sum_{t=0}^{T-1} \E\left[ \lVert \nabla f(x_t) \rVert^2 \right] &\leq \frac{f(x_0) - f^*}{T \eta}
    + 2 \eta L \noiseub
\end{align}
and therefore
\begin{align}
    \frac1T \sum_{t=0}^{T-1} \E\left[ \lVert \nabla f(x_t) \rVert^2 \right] &\leq \left( \frac{f(x_0) - f^*}{T \eta}
    + 2 \eta L \noiseub \right) \frac{2}{\lambda_-}.
\end{align}
Optimally choosing $\eta = \sqrt{(f(x_0) - f^*) / (2 T L \noiseub )}$ yields the overall bound 
\begin{align}
    \frac{1}{T} \sum_{t=0}^{T-1} \E\left[ \lVert \nabla f(x_t) \rVert^2 \right] 
    &\leq \sqrt{\frac{32 (f(x_0) - f^*) L \noiseub }{T}} \cdot \frac{1}{\lambda_-}.
\end{align}

Rephrasing, in order to be guaranteed that the left hand term is bounded by $\tau^2$, it suffices to choose $T$ so that
\begin{align}
    \sqrt{\frac{32 (f(x_0) - f^*) L \noiseub }{T}} \cdot \frac{1}{\lambda_-} \leq \tau^2 \\
    \Leftrightarrow T \geq  \frac{32 (f(x_0) - f^*) L \noiseub }{\tau^4 \lambda_-^2}
\end{align}
and
\begin{align}
    \eta &= \sqrt{\frac{f(x_0) - f^* }{ 2 T L \noiseub}}
    \leq \sqrt{\frac{ (f(x_0) - f^*) \tau^4 \lambda_-^{2}}{ 32 (f(x_0) - f^*) L^2 \noiseub ^2  }} 
    = \frac{ \tau^2 \lambda_-}{ 4\sqrt2 L \noiseub }.
\end{align}
\end{proof}


\section{Online Matrix Estimation}
\label{appendix:online-matrix-estimation}
We first reproduce the Matrix Freedman inequality as presented by~\citet{tropp2011freedman}:
\begin{theorem}[Matrix Freedman]
\label{theorem:matrix-freedman}
Consider a matrix martingale $\{Y_i : i = 0, 1, \dots\}$ (adapted to the filtration $\Fs_i$) whose values are symmetric $d\times d$ matrices, and let $\{Z_i : i = 1, 2, \dots\}$ be the difference sequence, i.e. $Z_i = Y_i - Y_{i-1}$.
For simplicity, let $Y_0 = 0$, so that $Y_n = \sum_{i=1}^n Z_i$.
Assume that $\lVert Z_i \rVert \leq R$ almost surely for each $i=1,2,\dots$.
Define $W_i := \sum_{j=1}^i \E[Z_j^2 | \Fs_{j-1 }]$. Then for all $k \geq 0$, 
\begin{align*}
    \mathbb P\left( \left\lVert Y_n \right\rVert \geq k \text{ and } \lVert W_n \rVert \leq \sigma^2 \right) \leq d \exp\left( \frac{-k^2/2}{\sigma^2 + Rk / 3} \right).
\end{align*}
\end{theorem}

\begin{corollary}
\label{corollary:matrix-freedman-variance}
Let $\{Z_i : i = 1, 2, \dots\}$ be a martingale difference sequence (adapted to the filtration $\Fs_i$) whose values are symmetric $d\times d$ matrices.
Assume $\lVert Z_i \rVert \leq R$ and $\lVert \E[Z_i^2 | \Fs_{i-1}] \rVert \leq \sigma_\mathrm{max}^2$ for all $i$.
Let $w \in \Delta_n$ in the simplex. Then for all $k \leq 3 \lVert w \rVert_2^2 \sigma_\mathrm{max}^2 / R$,
\begin{align*}
    \mathbb P\left( \left\lVert \sum_{i=1}^n w_i Z_i \right\rVert \geq k \right) \leq d \exp\left( \frac{-k^2}{4\lVert w \rVert_2^2 \sigma_\mathrm{max}^2 } \right).
\end{align*}
\end{corollary}
\begin{proof}
Observe that $Y_i := \sum_{j=1}^i w_j Z_j$ is a matrix martingale;
we are trying to bound $\mathbb P(\lVert Y_n \rVert \geq k)$.
Define the predictable quadratic variation process $W_i := \sum_{j=1}^i \E[(w_j Z_j)^2 | \Fs_{j-1}]$.
By assumption, we may bound
\begin{align}
	\lVert W_n \rVert 
	= \left\lVert \sum_{j=1}^n \E[w_j^2 Z_j^2 | \Fs_{j-1}] \right\rVert 
	\leq \sum_{j=1}^n \left\lVert \E[w_j^2 Z_j^2 | \Fs_{j-1}] \right\rVert 
	&= \sum_{j=1}^n w_j^2 \left\lVert \E[Z_j^2 | \Fs_{j-1}] \right\rVert \\
	&\leq \sum_{j=1}^n w_j^2 \sigma^2_\text{max} = \sigma^2_\text{max} \lVert w \rVert^2_2.
\end{align}
In other words, we can deterministically bound $\lVert W_n \rVert \leq \sigma^2_\text{max} \lVert w \rVert_2^2$.
Combining this bound with Theorem~\ref{theorem:matrix-freedman}, it follows that for any $k \geq 0$,
\begin{align}
	\mathbb P\left( \left\lVert Y_n \right\rVert \geq k \right)
	&= \mathbb P\left( \left\lVert Y_n \right\rVert \geq k \text{ and } \lVert W_n \rVert \leq \sigma^2_\text{max} \lVert w \rVert_2^2 \right) \\
	&\leq d \exp\left( \frac{-k^2/2}{\sigma^2_\text{max} \lVert w \rVert_2^2 + Rk / 3} \right).
\end{align}
By assumption, $k \leq 3 \lVert w \rVert_2^2 \sigma^2_\text{max} / R$, so $Rk/3 \leq \sigma^2_\text{max} \lVert w \rVert_2^2$, and we may further bound
\begin{equation*}
	d \exp\left( \frac{-k^2/2}{\sigma^2_\text{max} \lVert w \rVert_2^2 + Rk / 3} \right)
	\leq d \exp\left( \frac{-k^2}{4 \lVert w \rVert_2^2 \sigma^2_\text{max}} \right).
\end{equation*}
\end{proof}



Now we can apply the above matrix concentration results to prove Theorem~\ref{theorem:online-matrix-estimation}:
\begin{proof}[Proof of Theorem~\ref{theorem:online-matrix-estimation}]
First we separately bound the bias and variance of the estimate $\sum_{t=1}^T w_t Y_t$, then use Corollary~\ref{corollary:matrix-freedman-variance}.
Since $\E[Y_t | \Fs_{t-1}] = G_t = G(x_t)$, 
the bias of the estimate is:
\begin{align}
    \left\lVert \sum_{t=1}^T w_t G(x_t) - G(x_T) \right\rVert
    &= \left\lVert \sum_{t=1}^T w_t (G(x_t) - G(x_T)) \right\rVert \\
    &\leq \sum_{t=1}^T w_t \lVert G(x_t) - G(x_T) \rVert \\
    &\leq L \sum_{t=1}^T w_t \lVert x_t - x_T \rVert  \\
    &\leq L \sum_{t=1}^T w_t \sum_{s = t+1}^T \lVert x_s - x_{s-1} \rVert \\
    &\leq \eta M L \sum_{t=1}^T w_t (T - t) \\
    &= \eta M L \cdot \frac{1}{\sum_{t=1}^T \beta^{T-t}} \cdot \sum_{t=1}^T \beta^{T-t} (T - t). 
\end{align}
Note that by a well-known identity,
\begin{align}
    \sum_{t=1}^T \beta^{T-t} (T - t)
    = \sum_{s=0}^{T-1} s \beta^{s} 
    \leq \sum_{s=0}^{\infty} s \beta^{s} 
    = \frac{\beta}{(1-\beta)^2}.
\end{align}
Hence, the bias is bounded by
\begin{align}
    \eta M L \cdot \frac{1}{\sum_{t=1}^T \beta^{T-t}} \cdot \frac{\beta}{(1-\beta)^2}
    &= \eta M L \cdot \frac{1 - \beta}{1 - \beta^T} \cdot \frac{\beta}{(1-\beta)^2} \\
    &= \eta M L \cdot \frac{1}{1-\beta^T} \cdot \frac{\beta}{1-\beta} \\
    &\leq M L \cdot \frac{\eta}{(1-\beta)(1-\beta^T)}.
\end{align}

Applying Corollary~\ref{corollary:matrix-freedman-variance} to the martingale difference sequence $Z_t = Y_t - G(x_t)$,
we have that
\begin{align*}
    \mathbb P\left( \left\lVert \sum_{t=1}^T w_t (Y_t - G(x_t)) \right\rVert > k \right) 
    \leq d \exp\left( \frac{-k^2}{4\lVert w \rVert_2^2 \sigma_\mathrm{max}^2 } \right).
\end{align*}
Now note that
\begin{align}
    \lVert w \rVert_2^2 = \sum_{t=1}^T w_t^2
    &= \frac{1}{(\sum_{t=1}^T \beta^{T-t})^2} \sum_{t=1}^T (\beta^2)^{T-t} \\
    &= \frac{(1-\beta)^2}{(1-\beta^T)^2} \sum_{t=1}^T (\beta^2)^{T-t} \\
    &= \frac{(1-\beta)^2}{(1-\beta^T)^2} \cdot \frac{1-\beta^{2T}}{1-\beta^2} \\
    &= \frac{1-\beta^{2T}}{(1-\beta^T)^2} \cdot \frac{(1-\beta)^2}{1-\beta^2} \\
    &= \frac{1+\beta^{T}}{1-\beta^T} \cdot \frac{1-\beta}{1+\beta} \\
    &\leq \frac{2(1-\beta)}{1-\beta^T}.
\end{align}
Setting the right hand side of the high probability bound to $\delta$, we have concentration w.p. $1-\delta$ for $k$ satisfying
\begin{align}
    \delta &\geq d \exp\left( \frac{-k^2}{4\lVert w \rVert_2^2 \sigma_\mathrm{max}^2 } \right).
\end{align}
Rearranging, we find
\begin{align}
    &\log(d/\delta) \leq \frac{k^2}{ 4 \lVert w \rVert_2^2 \sigma_\mathrm{max}^2 } \\
    \Leftrightarrow \; &k \geq 2 \sigma_\mathrm{max} \lVert w \rVert_2 \sqrt{\log(d/\delta)}.
\end{align}

Combining this with the triangle inequality, 
\begin{align}
    \left\lVert \sum_{t=1}^T w_t y_t - G(x_T) \right\rVert
    &\leq \left\lVert \sum_{t=1}^T w_t y_t - \sum_{t=1}^T w_t G(x_t) \right\rVert + \left\lVert \sum_{t=1}^T w_t (G(x_t) - G(x_T)) \right\rVert \\
    &\leq 2 \sigma_\mathrm{max} \lVert w \rVert_2 \sqrt{\log(d/\delta)} + M L \cdot \frac{\eta}{(1-\beta)(1-\beta^T)} \\
    &\leq 2^{3/2} \sigma_\mathrm{max} \frac{\sqrt{1-\beta}}{\sqrt{1-\beta^T}} \sqrt{\log(d/\delta)}
     + M L \cdot \frac{\eta}{(1-\beta)(1-\beta^T)}.
\end{align}
with probability $1-\delta$. 
Since $1/\sqrt{1-\beta^T} \leq 1/(1-\beta^T)$, this can further be bounded by
\begin{align}
    \left( 2^{3/2} \sigma_\mathrm{max} \sqrt{1-\beta} \sqrt{\log(d/\delta)} + M L \cdot \frac{\eta}{(1-\beta)} \right) \cdot \frac{1}{1-\beta^T}.
\end{align}

Write $\alpha = 1-\beta$. The inner part of the bound is optimized when
\begin{align}
    2^{3/2} \sigma_\mathrm{max} \sqrt\alpha \sqrt{\log(d/\delta)} &= M L \cdot \frac{\eta}{\alpha} \\
    \Leftrightarrow
    \alpha^{3/2} &= \frac{M L \eta}{2^{3/2} \sigma_\mathrm{max} \sqrt{\log(d/\delta)}} \\
    \Leftrightarrow
    \alpha &= 
    \frac{M^{2/3} L^{2/3} \eta^{2/3}}{2 \sigma_\mathrm{max}^{2/3} (\log(d/\delta))^{1/3}}
\end{align}
for which the overall inner bound is
\begin{align}
    2 \cdot 2^{3/2} \sigma_\mathrm{max} \sqrt{\alpha} \sqrt{\log(d/\delta)}
    =
    4 \sigma_\mathrm{max}^{2/3} (\log(d/\delta))^{1/3} M^{1/3} L^{1/3} \eta^{1/3}.
\end{align}
If $T$ is sufficiently large, the $1/(1-\beta^T)$ term will be less than 2. In particular,
\begin{align}
    T > \frac{2}{\log(1+\alpha)} 
    \implies \frac{1}{1-(1-\alpha)^T} < 2.
\end{align}
Since $\log(1+\alpha) > \alpha/2$ for $\alpha < 1$, it suffices to have $T > 4 / \alpha$.
\end{proof}

\section{Converting Noise Estimates into Preconditioner Estimates}
\label{appendix:noise-estimates-preconditioner}

\begin{lemma}
\label{lemma:g-inverse-bound}
Suppose $\lVert G - \hat G \rVert \leq \varepsilon$, i.e. $\hat G$ is a good estimate of $G$ in operator norm. 
Assume $\varepsilon$ is so small that $\varepsilon \lVert G^{-1} \rVert < 1/2$.
Then,
\begin{align}
    \lVert G^{-1} - \hat G^{-1} \rVert \leq \frac{\varepsilon}{2 (\lambda_\mathrm{min}(G))^2}.
\end{align}
\end{lemma}
\begin{proof}
Observe
\begin{align}
	G^{-1} ( \hat G - G ) \hat G^{-1} = G^{-1} - \hat G^{-1}.
\end{align}
Therefore,
\begin{align}
	\delta = \lVert G^{-1} - \hat G^{-1} \rVert 
	&= \lVert G^{-1} ( \hat G - G ) \hat G^{-1} \rVert \\
	&\leq \varepsilon \lVert G^{-1} \rVert \lVert \hat G^{-1} \rVert \\
	&\leq \varepsilon \lVert G^{-1} \rVert (\lVert G^{-1} \rVert + \delta ).
\end{align}
Grouping $\delta$ terms together, we find
\begin{align}
	(1 - \varepsilon \lVert G^{-1} \rVert ) \delta &\leq \varepsilon \lVert G^{-1} \rVert^2 \\
	\implies \delta &\leq \frac{\lVert G^{-1} \rVert^2}{1 - \varepsilon \lVert G^{-1} \rVert} \cdot\varepsilon.
\end{align}
By assumption $\varepsilon$ is small enough so that $\varepsilon \lVert G^{-1} \rVert < 1/2$, so overall we have
\begin{align}
	\delta &\leq \frac{ \lVert G^{-1} \rVert^2}{2} \cdot\varepsilon = \frac{1}{2 (\lambda_\mathrm{min}(G))^2} \cdot \varepsilon.
\end{align}
\end{proof}

\begin{lemma}
\label{lemma:g-half-bound}
Suppose $\lVert G - \hat G \rVert \leq \varepsilon$, i.e. $\hat G$ is a good estimate of $G$ in operator norm.
Assume $\varepsilon$ is so small that $\varepsilon < \frac34 \lambda_\mathrm{min}(G)$.
Then,
\begin{align}
    \lVert G^{1/2} - \hat G^{1/2} \rVert \leq \frac{\varepsilon}{ ( \lambda_\mathrm{min}(G) )^{1/2}}.
\end{align}
\end{lemma}
\begin{proof}
We can equivalently write
\begin{align}
    G - \varepsilon I \preceq \hat G \preceq G + \varepsilon I.
\end{align}
By monotonicity of the matrix square root, 
\begin{align}
    ( G - \varepsilon I)^{1/2} \preceq \hat G^{1/2} \preceq (G + \varepsilon I)^{1/2}
\end{align}
and therefore
\begin{align}
    ( G - \varepsilon I)^{1/2} - G^{1/2}
    &\preceq \hat G^{1/2} - G^{1/2} \\
    &\preceq (G + \varepsilon I)^{1/2} - G^{1/2}.
\end{align}
At this point we can bound each side by applying Lemma~\ref{lemma:square-root-opnorm-bound} to $G$ and to $G - \varepsilon I$. The result is the bound
\begin{align*}
    \frac{-\varepsilon}{2 ( \lambda_\mathrm{min}(G) - \varepsilon )^{1/2}}
    \preceq \hat G^{1/2} - G^{1/2}
    \preceq \frac{\varepsilon}{2 (\lambda_\mathrm{min}(G))^{1/2}}.
\end{align*}
The lower bound is looser, so the operator norm of the difference is bounded by
\begin{align*}
    \frac{\varepsilon}{2 ( \lambda_\mathrm{min}(G) - \varepsilon )^{1/2}}
    < \frac{\varepsilon}{2 ( \frac14 \lambda_\mathrm{min}(G) )^{1/2}}
    = \frac{\varepsilon}{( \lambda_\mathrm{min}(G) )^{1/2}}.
\end{align*}
\end{proof}

\begin{lemma}
\label{lemma:square-root-opnorm-bound}
Let $A \succ 0$ and $\varepsilon > 0$. Then
\begin{align}
    \lVert (A + \varepsilon)^{1/2} - A^{1/2} \lVert \leq \frac{\varepsilon}{2 (\lambda_\mathrm{min}(A) )^{1/2}}.
\end{align}
\end{lemma}
\begin{proof}
The bound reduces to plugging in the eigenvalues of $A$ to a scalar function $f : \reals \to \reals$.
Define $f(x) = (x + \varepsilon)^{1/2} - x^{1/2}$. Note that
\begin{align}
    f(x)
    &= \frac{((x + \varepsilon)^{1/2} - x^{1/2}) ((x + \varepsilon)^{1/2} + x^{1/2})}{(x + \varepsilon)^{1/2} + x^{1/2}} \\
    &= \frac{(x+\varepsilon) - x}{(x + \varepsilon)^{1/2} + x^{1/2}} \\
    &= \frac{\varepsilon}{(x + \varepsilon)^{1/2} + x^{1/2}} \\
    &\leq \frac{\varepsilon}{2 x^{1/2}},
\end{align}
from which the result follows.
\end{proof}


\begin{corollary}
\label{corollary-var1-preconditioner-opnorm}
Suppose $\lVert G - \hat G \rVert \leq \varepsilon$, for small enough $\varepsilon$. Then,
\begin{align*}
    \lVert (G+\delta I)^{-1/2} - (\hat G + \delta I)^{-1/2} \rVert \leq \frac{\varepsilon}{2 (\delta + \lambda_\mathrm{min}(G))^{3/2}}.
\end{align*}
\end{corollary}
\begin{proof}
Simply apply Lemma~\ref{lemma:g-inverse-bound} and Lemma~\ref{lemma:g-half-bound} to $G + \delta I$.
\end{proof}


\end{document}